\documentclass[journal,onecolumn]{IEEEtran}

\usepackage{amsmath,amssymb,amsbsy,amsfonts,amscd,bm,url,color,latexsym}
\usepackage[hidelinks]{hyperref}
\usepackage{algorithm}
\usepackage{algorithmic}
\usepackage{comment}
\usepackage{multirow}
\usepackage{enumitem}
\usepackage{fancyhdr}
\usepackage{tcolorbox}
\usepackage{wrapfig}
\usepackage{cleveref}
\usepackage{subcaption}
\usepackage{authblk}
\newtheorem{lemma}{Lemma}%[section] %%    with section number.
\newtheorem{cor}{Corollary}%[section]
\newtheorem{assum}{Assumption}

\newtheorem{theorem}{Theorem}

\newcommand{\R}{\mathbb{R}}
\def \real    { \mathbb{R} }

\newcommand{\e}{\begin{equation}}
\newcommand{\ee}{\end{equation}}
\newcommand{\en}{\begin{equation*}}
\newcommand{\een}{\end{equation*}}
\newcommand{\eqn}{\begin{eqnarray}}
\newcommand{\eeqn}{\end{eqnarray}}
\newcommand{\bmat}{\begin{bmatrix}}
\newcommand{\emat}{\end{bmatrix}}
\DeclareMathAlphabet\mathbfcal{OMS}{cmsy}{b}{n}
\newcommand{\vct}[1]{\boldsymbol{#1}}
\newcommand{\mtx}[1]{\boldsymbol{#1}}
\newcommand{\<}{\langle}
\renewcommand{\>}{\rangle}
\def \vec       {\operatorname*{vec}}
\newcommand{\wt}{\widetilde}
\newcommand{\ol}{\overline}

\newcommand{\calN}{\mathcal{N}}

\newcommand{\va}{\vct{a}}

\newcommand{\vs}{\vct{s}}

\newcommand{\vy}{\vct{y}}

\newcommand{\vtheta}{\vct{\theta}}

\newcommand{\vphi}{\vct{\phi}}

\newcommand{\mA}{\mtx{A}}

\newcommand{\mW}{\mtx{W}}
\newcommand{\mX}{\mtx{X}}
\newcommand{\mY}{\mtx{Y}}
\newcommand{\mZ}{\mtx{Z}}

\newcommand{\mTheta}{\mtx{\Theta}}

\newlength{\imgwidth}
\hypersetup{
    colorlinks=true,%
    citecolor=black,%
    filecolor=black,%
    linkcolor=black,%
    urlcolor=black
}
\renewcommand{\mathbf}{\boldsymbol}

\def \endprf{\hfill {\vrule height6pt width6pt depth0pt}\medskip}
\newenvironment{proof}{\noindent {\bf Proof} }{\endprf\par}

\ifCLASSINFOpdf
\else
   \usepackage[dvips]{graphicx}
\fi
\usepackage{url}

\usepackage{graphicx}

\usepackage[utf8]{inputenc} % allow utf-8 input
\usepackage[T1]{fontenc}    % use 8-bit T1 fonts
\usepackage{hyperref}       % hyperlinks
\usepackage{url}            % simple URL typesetting
\usepackage{booktabs}       % professional-quality tables
\usepackage{amsfonts}       % blackboard math symbols
\usepackage{nicefrac}       % compact symbols for 1/2, etc.
\usepackage{microtype}      % microtypography
\usepackage{xcolor}         % colors
\usepackage{epsfig}
\usepackage{epstopdf}
\usepackage{wrapfig}

\usepackage{balance}

\usepackage{bm}

\usepackage[numbers,sort&compress]{natbib}

\usepackage{array}
\usepackage{diagbox}
\usepackage{pdfsync}

\makeatletter
\def\@IEEEsectpunct{\ \,}
\def\paragraph{\@startsection{paragraph}{4}{\z@}{1.5ex plus 1.5ex minus 0.5ex}%
{0ex}{\bfseries}}
\makeatother

\begin{document}

\title{On the Convergence of Gradient Descent on Learning Transformers with Residual Connections}

\author{Zhen Qin,  Jinxin Zhou, Jiachen Jiang, and Zhihui Zhu, \IEEEmembership{Member, IEEE}
\thanks{
Zhen Qin is with Michigan Institute for Computational Discovery and Engineering, Department of Electrical Engineering and Computer Science, Department of Statistics, University of Michigan, Ann Arbor, MI 48109 USA, and also with the Department of Computer Science and Engineering, The Ohio State University, Columbus, OH 43210 USA (e-mail: zhenqin@umich.edu).

Jinxin Zhou, Jiachen Jiang and Zhihui Zhu are with the Department of Computer Science and Engineering, the Ohio State University, OH 43210, USA (e-mail: \{zhou.3820, jiang.2880, zhu.3440\}@osu.edu).

This work was supported by NSF grants ECCS-2409701 and IIS-2402952.
}
}

%\markboth{Journal of \LaTeX\ Class Files, Vol. 14, No. 8, August 2015}
%{Shell \MakeLowercase{\textit{et al.}}: Bare Demo of IEEEtran.cls for IEEE Journals}
\maketitle

\begin{abstract}
Transformer models have emerged as fundamental tools across various scientific and engineering disciplines, owing to their outstanding performance in diverse applications.
Despite this empirical success, the theoretical foundations of Transformers remain relatively underdeveloped, particularly in understanding their training dynamics. Existing research predominantly examines isolated components--such as self-attention mechanisms and feedforward networks--without thoroughly investigating the interdependencies between these components, especially when residual connections are present. In this paper, we aim to bridge this gap by analyzing the convergence behavior of a structurally complete yet single-layer Transformer, comprising self-attention, a feedforward network, and residual connections. We demonstrate that, under appropriate initialization, gradient descent exhibits a linear convergence rate, where the convergence speed is determined by the minimum and maximum singular values of the output matrix from the attention layer. Moreover, our convergence analysis establishes a theoretical characterization of residual connections by showing that they alleviate the ill-conditioning of the attention output matrix, which arises from the low-rank structure induced by the softmax operation, thereby improving optimization stability. Empirical results corroborate our theoretical insights, illustrating the beneficial role of residual connections in promoting convergence stability.
\end{abstract}

\begin{IEEEkeywords}
Transformer, residual connections,  feedforward network, training dynamics, stability.
\end{IEEEkeywords}

\IEEEpeerreviewmaketitle

\section{Introduction}

\IEEEPARstart{T}RANSFORMER model architectures \cite{vaswani2017attention} have gained widespread recognition and popularity in various scientific and engineering applications, consistently achieving outstanding performance across numerous domains. Notably, Transformers have achieved substantial success in natural language processing tasks \cite{radford2019language}, recommendation systems \cite{zhou2018deep}, reinforcement learning \cite{chen2021decision}, computer vision \cite{dosovitskiy2020image}, multi-modal signal processing \cite{tsai2019multimodal}, quantum information \cite{ma2025tomography}, and communication systems \cite{kim2023transformer}. One prominent example is their impressive performance in large language models such as GPT-4 \cite{achiam2023gpt}, where Transformers play a central role in achieving unprecedented language generation capabilities. However, despite their empirical success, the theoretical understanding of Transformers remains limited, posing significant challenges in analyzing their fundamental mechanisms and performance guarantees.

To address this gap, an expanding body of theoretical literature has explored various facets of Transformer models, including the impact of initialization \cite{makkuva2024local}, sample complexity guarantees \cite{ildiz2024self}, scaling limits \cite{bordelon2024infinite}, and implicit regularization effects \cite{ataee2023max}. A crucial research direction emerging from these efforts is the investigation of training dynamics in Transformers, with particular emphasis on the two fundamental components within each layer: the feedforward network \cite{arora2018convergence,nguyen2020global,chatterjee2022convergence,qin2024convergence} and the self-attention mechanism \cite{zhang2024trained,huang2023context,yang2024context,zhang2025training,li2024training,song2024unraveling}. Existing studies have provided important insights into the optimization landscape and convergence behavior of these modules, establishing guarantees under various model and data assumptions. Notably, most theoretical investigations analyze these components independently, typically characterizing their properties in isolation. While such analyses have substantially advanced the understanding of individual modules, they often overlook the intricate interactions between self-attention, feedforward network, and architectural elements such as residual connections. As a result, the extent to which theoretical guarantees derived for isolated components extend to the full Transformer architecture remains insufficiently understood.

Beyond investigating individual components, only recent work \cite{wu2023convergence} studies the convergence rate of a softmax attention layer combined with a feedforward network using a single ReLU activation function, without imposing any specific restrictions on the input matrix.  This work shows that, with proper initialization, gradient descent can achieve a linear convergence rate. Moreover, it demonstrates that under suitable conditions, commonly used initialization schemes such as LeCun and He, and NTK initialization satisfy the requirements for convergence.
\begin{figure}[!t]
    \centering
    \includegraphics[width=8.5cm]
    {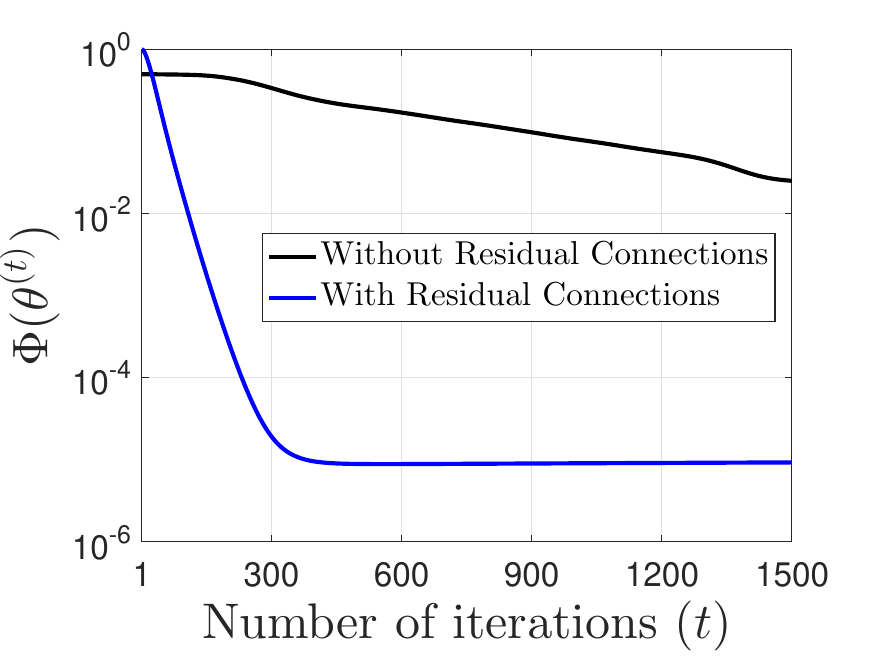}
    \caption{Comparison of training dynamics of single-layer Transformers with and without residual connections.}
    \label{fig:comparison_with_without_residual}
\end{figure}
However, a key limitation of this analysis is the assumption that the weight matrices in the feedforward network are identity matrices, which restricts its applicability to more general settings. Consequently, the results only partially capture the training dynamics of a single-layer Transformer.
Furthermore, while previous works have advanced the theoretical understanding of residual connections--such as \cite{hardt2016identity}, which proved that deep linear residual networks, and \cite{liu2019towards}, which showed that a two-layer non-overlapping convolutional residual network does not suffer from spurious local optima due to the use of residual connections, and \cite{scholkemper2024residual}, which demonstrated that residual connections alleviate the oversmoothing problem in graph neural networks—the theoretical understanding of their role in Transformers remains underdeveloped. To motivate further investigation, we illustrate in \Cref{fig:comparison between with and without residual} that even a single-layer Transformer benefits from residual connections, exhibiting a faster convergence rate compared to its counterpart without them. This phenomenon can be intuitively explained by the occurrence of rank collapse \cite{noci2022signal} in certain scenarios, leading to an ill-conditioned output matrix in the softmax attention layer, which, in turn, may hinder the convergence speed during Transformer training.
In this regard, residual connections could help mitigate such issues by stabilizing the training process. To the best of our knowledge, the convergence behavior of Transformers jointly involving self-attention, feedforward networks, and residual connections has not yet been theoretically characterized.

{\bf Our contribution:} In this paper, we analyze the convergence behavior of gradient descent for a single-layer Transformer that jointly integrates single-head self-attention, a feedforward network, and residual connections. We demonstrate that, with proper initialization, gradient descent achieves a linear convergence rate. In addition, our analysis establishes that the linear convergence rate is governed by the extreme singular values of the attention output matrix. This result naturally leads to a convergence-theoretic interpretation of residual connections, showing that they mitigate the ill-conditioning and thereby improve optimization behavior. Experimental results validate our theoretical findings, confirming the linear convergence rate of gradient descent and demonstrating the beneficial effect of introducing residual connections.

\subsection{Notation} We use  bold capital letters (e.g., $\mY$) to denote matrices,  bold lowercase letters (e.g., $\vy$) to denote vectors, and italic letters (e.g., $y$) to denote scalar quantities.  Elements of matrices are denoted in parentheses, as in Matlab notation. For example, $\mY(s_1, s_2)$ denotes the element in position
$(s_1, s_2)$ of the matrix $\mY$.
$\|\mX\|$ and $\|\mX\|_F$ respectively represent the spectral norm and Frobenius norm of $\mX$.
$\sigma_{i}(\mX)$ is the $i$-th singular value of $\mX$ and the condition number $\mX$ is defined as $\kappa(\mX) = \frac{\|\mX\|}{\sigma_{\min}(\mX)}$. For two positive quantities $a,b\in \real$, the inequality $b\lesssim a$ or $b = O(a)$ means $b\leq c a$ for some universal constant $c$; likewise, $b\gtrsim a$ or $b = \Omega(a)$ represents $b\ge ca$ for some universal constant $c$.

\section{Problem Setting}
\label{sec: problem setting}

This section formalizes the architecture of the Transformer model and the corresponding training objective. We consider a simplified setting involving a single-layer Transformer equipped with single-head self-attention, a feedforward network and residual connections. Given an input matrix $\mX \in \R^{M \times d}$, where $M$ denotes the number of discrete tokens and $d$ is the embedding dimension, the model is specified as follows:
\begin{eqnarray}
    \label{transformer structure}
    F_{\mTheta}(\mX) = (\text{FFN}(\text{Attn}(\mX) + \mX) + \text{Attn}(\mX) + \mX)\mW_U,
\end{eqnarray}
where $\mTheta = \{\mW_1, \mW_2,\mW_Q, \mW_K, \mW_V, \mW_U \}$. For analytical simplicity, layer normalization is not incorporated in our formulation, aligning with prior convergence analyses that likewise omit it in \cite{song2024unraveling, wu2023convergence}. We now provide a detailed introduction to each component of the model as follows.
\begin{itemize}

\item \textbf{Self-Attention Mechanism}: A self-attention mechanism is a central component of the Transformer architecture, enabling contextualization of token representations across layers. We denote the attention head function by $\text{Attn}(\cdot)$. In this work, we consider the standard softmax attention \cite{vaswani2017attention}, for which $\text{Attn}(\cdot)$ is defined as follows:
    \begin{eqnarray}
    \label{self attention model}
    \text{Attn}(\mX) :=  \phi_s\bigg(\frac{\mX\mW_Q\mW_K^\top \mX^\top}{\sqrt{d_{QK}}}\bigg)\mX\mW_V \in\R^{M\times d},
    \end{eqnarray}
    where $\mW_Q \in\R^{d \times d_{QK}}$, $\mW_K\in\R^{d \times d_{QK}}$ and $\mW_V\in\R^{d\times d}$ denote the query, key, and value weight matrices, respectively. The function $\phi_s(\cdot)$ denotes the row-wise softmax operation.

\item \textbf{Feedforward Network}: The feedforward network (FFN) in each Transformer block comprises two learnable weight matrices, $\mW_1\in\R^{d\times d_1}$ and $\mW_2\in\R^{d_1\times d}$. In line with recent architectures such as LLaMA \cite{touvron2023llama}, PaLM \cite{chowdhery2023palm}, and OLMo \cite{groeneveld2024olmo}, we omit bias terms. The feedforward function $\text{FFN}(\cdot)$ is defined as
    \begin{eqnarray}
    \label{definition FFN model}
    \text{FFN}(\mZ(\mX)) := \phi_r(\mZ(\mX)\mW_1)\mW_2 \in\R^{M\times d},
    \end{eqnarray}
    where $\mZ(\mX) = \text{Attn}(\mX) + \mX$ and $\phi_r$ is a general element-wise activation function.

\item \textbf{Residual Connections}: Residual connections are incorporated to promote stable training and improve gradient flow in deep Transformer architectures. For each sub-layer (e.g., self-attention or feedforward), the residual connection adds the input (denoted by $\mX$ or $\mZ(\mX)$) of the sub-layer to its output.

\item \textbf{Unembedding Layer}: In the final layer of the Transformer, the residual stream is projected into the vocabulary space via the unembedding matrix $\mW_U \in \R^{d \times N}$. Without loss of generality, we include the unembedding matrix; for the convergence analysis of intermediate layers, this matrix can be treated as a constant $1$.
\end{itemize}

For clarity, the full sequence of the operations described above is illustrated in \Cref{Figure of single layer transformer}.
\begin{figure}[!h]
\centering
\includegraphics[width=12.5cm, keepaspectratio]%
{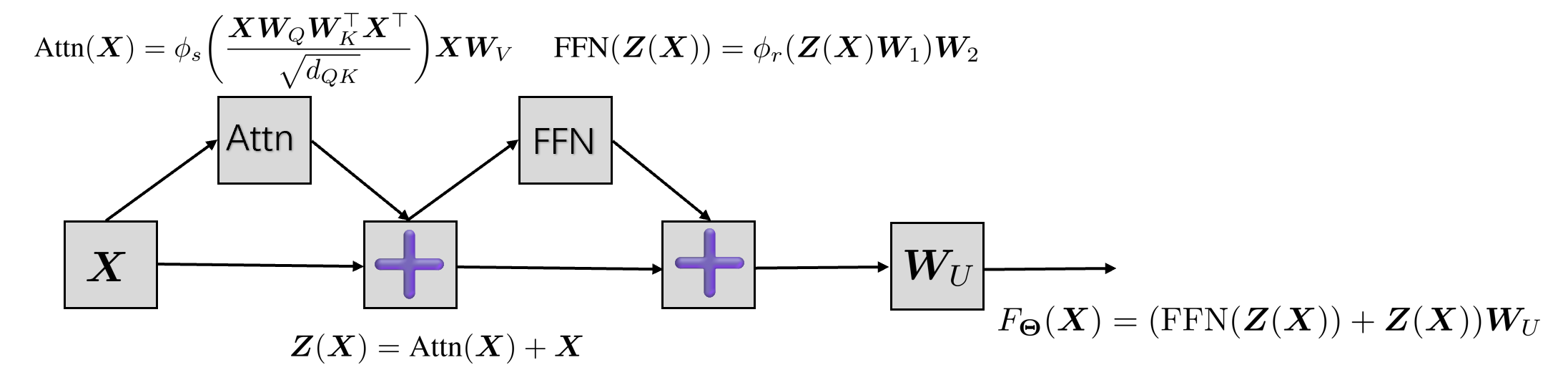}
\caption{Illustration of a single-layer Transformer with single-head attention, a feedforward network and residual connections.}
\label{Figure of single layer transformer}
\end{figure}

Based on the Transformer model described above, we consider the supervised learning setting with a dataset $\{\mX_p, \mY_p\}_{p=1}^P\in\{\R^{M \times d},\R^{M \times N} \} $. The training objective is to minimize the following squared Frobenius norm loss:
\begin{eqnarray}
    \label{loss function in the transformer structure}
     &\!\!\!\!\!\!\!\!&L(\mW_1, \mW_2,\mW_Q, \mW_K, \mW_V, \mW_U)\nonumber\\
      &\!\!\!\! = \!\!\!\!&\frac{1}{2}\sum_{p=1}^P\|F_{\mTheta}(\mX_p) - \mY_p \|_F^2 = \frac{1}{2}\|\ol F_{\mTheta}(\mX) - \ol \mY \|_F^2,
\end{eqnarray}
where $\ol F_{\mTheta}(\mX) = \begin{bmatrix} F_{\mTheta}(\mX_1)^\top \cdots F_{\mTheta}(\mX_p)^\top  \end{bmatrix}^\top\in\R^{MP\times N}$ and $\ol \mY = \begin{bmatrix} \mY_1^\top \cdots \mY_P^\top  \end{bmatrix}^\top\in\R^{MP\times N}$.

\section{Convergence Analysis}
\label{sec: convergence analysis}

In this section, we provide a theoretical analysis of the optimization method associated with problem \eqref{loss function in the transformer structure}.  Specifically, we investigate the convergence behavior of the vanilla gradient descent (GD) algorithm\footnote{To streamline the notation, , let us represent $\nabla_{\mW_b} L(\mW_1^{(t)}, \mW_2^{(t)}, \mW_Q^{(t)},\\  \mW_K^{(t)},  \mW_V^{(t)},  \mW_U^{(t)})$ as $\nabla_{\mW_b^{(t)}} L$.}, which updates each parameter according to
\begin{eqnarray}
    \label{every gradient descent}
    \mW_b^{(t+1)} = \mW_b^{(t)} - \mu \nabla_{\mW_b^{(t)}} L,
\end{eqnarray}
where $b \in \{ 1,2,Q,K,V,U\}$ and $\mu$ is the learning rate. All detailed gradient expressions are presented in {Appendix}~\ref{sec:summary of gradients}.
To facilitate our analysis, we introduce the following assumption on the Lipschitz condition of the activation functions.

\begin{assum}[Lipschitz condition of the activation function]
\label{assumption of activation function}
Let $\sigma_r$ be a non-decreasing function satisfying  $|\sigma_r(x) - \sigma_r(y)| \leq |x - y|$ for every $x, y \in \mathbb{R}$.
\end{assum}

This condition holds for several commonly used activation functions. For instance, the ReLU function $\phi_r(x) = \max\{0, x\}$ satisfies this assumption. Moreover, smooth approximations of ReLU, such as the Gaussian-smoothed ReLU discussed in~\cite{nguyen2020global}, also comply with this property.

Having established the above assumptions, we now vectorize the model output $\ol F_{\mTheta}(\mX)$ and the ground truth matrix $\ol\mY$ to simplify notation and streamline the analysis. Specifically, we let $f_{\vtheta}(\mX)   = \vec(\ol F_{\mTheta}(\mX))$ and $\vy =  \vec(\ol\mY)$, where the parameter vector $\vtheta$ is constructed by stacking the vectorized weights as $\vtheta = [\vec(\mW_1)^\top \ \vec(\mW_2)^\top \ \vec(\mW_U)^\top $ $   \vec(\mW_V)^\top \ \vec(\mW_Q)^\top \ \vec(\mW_K)^\top ]^\top$. Using this compact formulation, the original objective~\eqref{loss function in the transformer structure} can be equivalently rewritten as a standard least-squares loss:
\begin{eqnarray}
    \label{loss function of transformer another}
    \Phi(\vtheta) = \frac{1}{2}\|f_{\vtheta}(\mX) - \vy \|_2^2.
\end{eqnarray}
This reformulation allows us to analyze the convergence behavior of GD in a more tractable form, leveraging standard tools from vector-valued function analysis. Now, we present our convergence analysis as follows.
\begin{theorem}
\label{Theorem of convergence analysis of transformer model simplified}
Given a dataset $\{\mX_p, \mY_p\}_{p=1}^P$, we consider a single-layer Transformer architecture in \eqref{transformer structure}.   We define $\mZ^{(0)}(\mX_p) = \phi_s(\mX_p\mW_Q^{(0)}{\mW_K^{(0)}}^\top \mX_p^\top/\sqrt{d_{QK}})\mX_p\mW_V^{(0)} + \mX_p$, $\ol \lambda = \max_{b} \|\mW_b^{(0)}\|$ and  $\underline{\lambda} = \min_{b}  \sigma_{\min}(\mW_b^{(0)})$ where $b\in\{1,2,U,V,Q,K\}$. In addition, we assume  that the initialization matrices $\{\mW_b^{(0)} \}$ are either full row rank or full column rank, and are properly initialized\footnote{For the detailed initialization requirements, refer to \eqref{requirement of the initialization transformer theorem  detailed}.}.
Define $\vphi_P^{(t)} = \begin{bmatrix}(\phi_r(\mZ^{(t)}(\mX_1)\mW_1^{(t)}))^\top \cdots   (\phi_r(\mZ^{(t)}(\mX_P)\mW_1^{(t)}))^\top   \end{bmatrix}$ and $\alpha = \frac{\sigma_{\min}^2(\mW_U^{(0)})\sigma_{\min}^2(\vphi_P^{(0)})}{16}$. Using the gradient descent in \eqref{every gradient descent}, we have
\begin{eqnarray}
    \label{convergence analysis of GD transformer in the theorem detailed main paper}
    \Phi(\vtheta^{(t+1)}) \leq   (1 - \mu \alpha)\Phi(\vtheta^{(t)}),
\end{eqnarray}
where the learning rate satisfies $\mu \leq \min\{ \frac{1}{C}, \frac{1}{\alpha}\}$, and the constant $C$ is given by $C = \wt C \ol \lambda^5 \max_{p}\|\mX_p\|^3 \max_{p}\|\mZ^{(0)}(\mX_p)\|$ with $\wt C$ being a positive constant.
\end{theorem}
According to Theorem~\ref{Theorem of convergence analysis of transformer model simplified}, under suitable initialization, the GD algorithm described in~\eqref{every gradient descent} exhibits a linear convergence rate. In contrast to existing works such as \cite{wu2023convergence,song2024unraveling}, which analyze simplified architectures, our results rigorously characterize the convergence dynamics of a single-layer Transformer that jointly integrates single-head self-attention, a feedforward layer, and residual connections. The analysis requires a unified treatment of these components, rather than a straightforward combination of separate module-wise analyses; full technical details are provided in {Appendix}~\ref{proof of convergence rate in the tranformer model simplified}. It is worth emphasizing that our focus lies in establishing the local convergence behavior of GD. The initialization strategies, which have been extensively studied in prior literature (e.g., \cite{nguyen2020global,wu2023convergence}), are beyond the scope of this work.

\subsection{Global Minimum}

Next, we illustrate that the solution can converge to a global minimum, denoted as $\vtheta^\star$, satisfying $\Phi(\vtheta^\star) = 0$. This result highlights the theoretical guarantee of achieving optimal parameter estimation in Transformer models through gradient descent.
\begin{cor}
\label{Theorem of convergence analysis for theta of transformer model}
Under the same setting as in \Cref{Theorem of convergence analysis of transformer model simplified}, utilizing gradient descent, the following convergence rate holds:
\begin{eqnarray}
    \label{convergence rate of theta transformer main conclusion}
    \|\vtheta^{(k)} - \vtheta^\star\|_2 \leq (1 - \mu\alpha)^{\frac{k}{2}}\frac{ C_W}{\alpha} \Phi^{\frac{1}{2}}(\vtheta^{(0)}).
\end{eqnarray}
where the learning rate satisfies $\mu \leq \min\{ \frac{1}{C}, \frac{1}{\alpha}\}$ in which $C$  and $\alpha$ have been defined in \Cref{Theorem of convergence analysis of transformer model simplified}. In addition, we define $C_W = O(\ol \lambda\max_{p}\|\mZ^{(0)}(\mX_p)\| + (1 + \ol\lambda^2)(\max_{p}\|\mZ^{(0)}(\mX_p)\| + \ol \lambda \max_{p}\|\mX_p\|) + \max_{i,p}\|\mX_p(i,:)\|_2\|\mX_p\|^2(1+ \ol\lambda^2)\ol\lambda^3) $.
\end{cor}
The proof has been provided in {Appendix}~\ref{appendix:convergence to the global minimum}. This result guarantees convergence to the global minimum, given that the step size $\mu$ is properly chosen and the initialization conditions are satisfied. Specifically, \eqref{convergence rate of theta transformer main conclusion} highlights that as the number of iterations $k$ increases, the parameter estimate $\vtheta^{(k)}$ converges to $\vtheta^\star$ at a linear convergence rate.

\subsection{A Convergence-Theoretic Perspective on Residual Connections}
\label{sec: The importance of residual connections}

Building upon the preceding convergence analysis, we first demonstrate how the convergence rate is determined by the minimum and maximum singular values of the attention-layer output matrix. We then provide a convergence-theoretic perspective on residual connections in Transformers, highlighting their role in mitigating the ill-conditioning of the attention-layer output matrix.

To facilitate the subsequent derivation, we assume that the entries of the weight matrices $\mW_1^{(0)}$ and $\mW_U^{(0)}$ are independently drawn from Gaussian distributions with zero mean and variances $\gamma_1^2$ and $\gamma_U^2$, respectively, i.e., $\mW_1^{(0)}(i,j) \sim \calN(0, \gamma_1^2)$ and $\mW_U^{(0)}(i,j) \sim \calN(0, \gamma_U^2)$. In addition, we adopt the ReLU activation function and impose the conditions $d_1/4 \geq d$ and $N/4 \geq d$. Under these settings, with high probability, the lower bound of $\alpha$ can be expressed as
\begin{eqnarray}
    \label{lower bound of alpha}
    \alpha    &\!\!\!\! \gtrsim \!\!\!\!& \frac{d_1\gamma_1^2\gamma_U^2(\mu_1(\sigma_r))^2}{128}(\frac{\sqrt{N}}{2} - \sqrt{d})^2\nonumber\\
    &\!\!\!\! \!\!\!\!& \cdot(\frac{\sqrt{d_1}}{2} - \sqrt{d})^2\min_{p}\sigma_{\min}^2(\mZ^{(0)}(\mX_p)),
\end{eqnarray}
where $\mu_1(\sigma_r)$ denotes the first Hermite coefficient of the ReLU function, which satisfies $\mu_1(\sigma_r) > 0$. Specifically, the first inequality follows Weyl's inequality and the second inequality follows from the singular value inequality for Gaussian matrices, as stated in \Cref{spectral norm of Gaussian matrix inequality} of Appendix~\ref{Technical tools used in proofs}, along with results from \cite[Lemma C.2]{nguyen2020global} and \cite[Lemma 5.3]{nguyen2021tight}. Specifically, the inequality $\sigma_{\min}^2(\vphi_P^{(0)})\geq \sigma_{\min}^2((\phi_r(\mZ^{(0)}(\mX_p)\mW_1^{(0)}))) - \sum_{a\neq b}\|\phi_r(\mZ^{(0)}(\mX_p)\mW_1^{(0)}) (\phi_r(\mZ^{(0)}(\mX_p)\mW_1^{(0)}))^\top\|  $ ensures the bound holds. Nevertheless, under appropriate initialization, we ensure $\alpha > 0$, and thus focus primarily on the leading term in the final bound.

First, we examine from a theoretical perspective how the convergence rate is influenced by the attention layer.
When the weight matrices are initialized with Gaussian distributions, their spectral norms and minimum singular values can be bounded, as established in \cite[Lemma 4]{qin2025convergence}. In this case, the positive values $\ol \lambda$, $\underline{\lambda}$ and $\|\mX_p\|$ exert the same influence on the convergence analysis for Transformer model regardless of the presence of residual connections. Consequently, leveraging the lower bound in \eqref{lower bound of alpha}, when $\frac{1}{C} < \frac{1}{\alpha}$, the convergence rate $1 - \mu \alpha$ is predominantly influenced by the term $\frac{\min_{p}\sigma_{\min}^2(\mZ^{(0)}(\mX_p))}{\max_{p}\|\mZ^{(0)}(\mX_p) \|}  $. Here $\mZ^{(0)}(\mX_p) = \phi_s\big(\frac{\mX_p\mW_Q^{(0)}{\mW_K^{(0)}}^\top \mX_p^\top}{\sqrt{d_{QK}}}\big)\mX_p\mW_V^{(0)} + \mX_p$ is determined by the residual associated with the input data $\mX_p$. Thus, the term $1 - C_1\frac{\min_{p}\sigma_{\min}^2(\mZ^{(0)}(\mX_p))}{\max_{p}\|\mZ^{(0)}(\mX_p) \|}$ demonstrates how the convergence rate is determined by the properties of the attention-layer output and quantifies the influence of residual connections.

Next, we illustrate the role of residual connections in Transformers from a convergence-theoretic perspective.
Building on the above convergence-rate formulation, we now turn to extreme scenarios where the attention output may become ill-conditioned, in order to highlight how residual connections mitigate such issues. From a theoretical perspective, it is essential to recognize that when the constant $C_1$ is fixed, there may exist extreme cases in which $\text{Attn}(\mX)$ becomes ill-conditioned, potentially compromising the model's stability and convergence. For instance, as established by \cite[Theorem A.2]{noci2022signal}, consider the scenario where $d<\infty$ is fixed and $d_{QK}\to \infty$. In this regime, for any input $\mX$, it can be shown that $\text{Attn}(\mX)\to \frac{1}{M}\mathbf{1}_{M \times M}\mX\mW_V$, where $\mathbf{1}_{M \times M}$ denotes an $M\times M$ matrix with all elements equal to one. Consequently, in the absence of a residual connection (i.e., $\mZ^{(0)}(\mX) = \text{Attn}(\mX)$), the resulting matrix $\mZ^{(0)}(\mX)$ is rank-one, yielding $\sigma_{\min}(\mZ^{(0)}(\mX)) \to 0$, implying that the convergence curve will remain flat. However, when the residual connection is introduced,  the resulting matrix $\mZ^{(0)}(\mX)\to   \frac{1}{M}\mathbf{1}_{M \times M}\mX\mW_V + \mX$ maintains full rank as long as the input matrix $\mX$ is of full rank, thereby ensuring that the minimum singular value $\sigma_{\min}(\mZ^{(0)}(\mX))$ is strictly positive. {\it This enhancement in the rank structure fundamentally guarantees a reduction in the risk of convergence stagnation, thereby theoretically demonstrating the crucial role of residual connections in maintaining the stability and effectiveness of Transformers under extreme conditions.}

\section{Experimental Results}
\label{sec:experiment result}

In this section, we investigate the impact of residual connections on the training dynamics using real-world data.

\begin{figure}[!ht]
\centering

\begin{subfigure}[t]{0.43\textwidth}
    \centering
    \includegraphics[width=\linewidth]{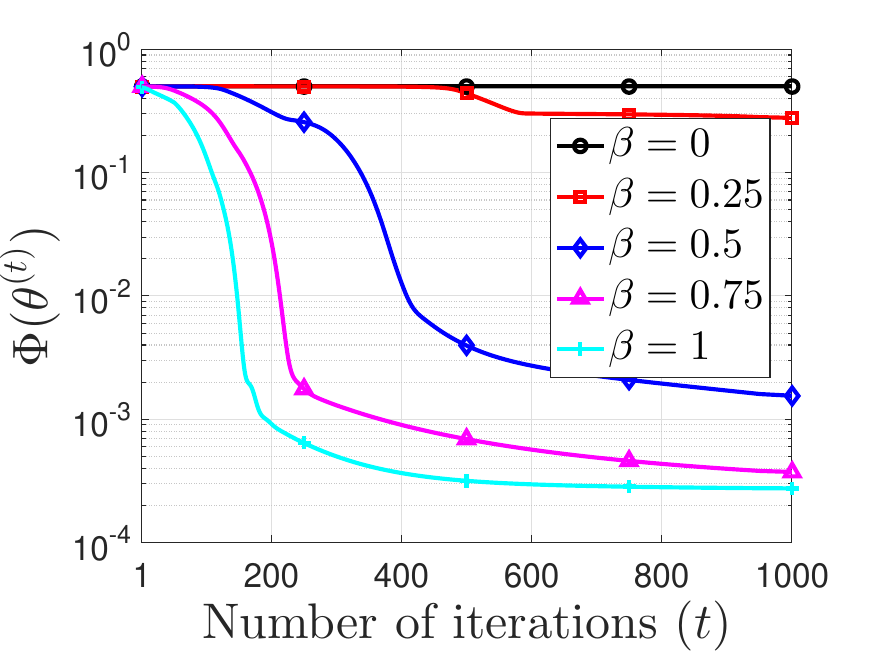}
    \caption{}
    \label{diff_beta_L1}
\end{subfigure}
\hfill
\begin{subfigure}[t]{0.43\textwidth}
    \centering
    \includegraphics[width=\linewidth]{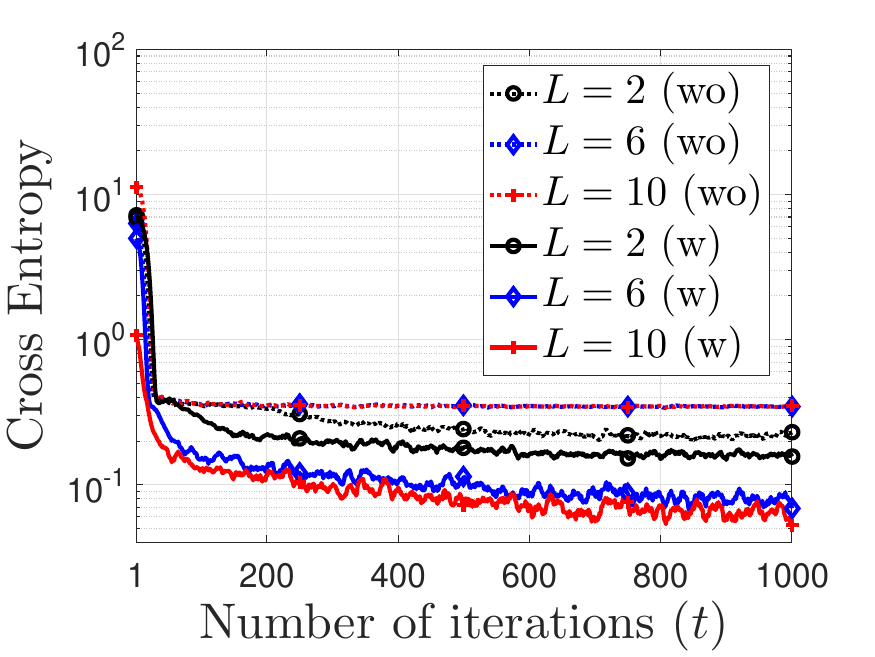}
    \caption{}
    \label{diff_language_model}
\end{subfigure}

\caption{
Training dynamics of Transformer models with and without residual connections:
(a) one-layer Transformers with varying residual coefficients $\beta$;
(b) $L$-layer Transformers trained with (w) and without (wo) residual connections.
}
\label{Figure of summary}
\end{figure}

To this end, we conduct our first experiment on the Jena Climate Dataset~\cite{attri2020timeseries}, a multivariate time-series benchmark that has been widely adopted in recent studies~\cite{duong2023dynamic,victor2024enhancing}. This experiment is designed to evaluate the role of residual connections within the modeling framework considered in this work. The specific experimental setup follows that of~\cite[Section~5]{qin2025convergence}. Drawing inspiration from \cite{zhang2024residual}, where a residual coefficient $\beta$ is introduced to construct the model as
$F_{\mTheta}(\mX) = (\text{FFN}(\text{Attn}(\mX) + \beta\mX) + \beta(\text{Attn}(\mX) + \beta\mX))\mW_U$, it was demonstrated that setting
$\beta$ smaller than $1$ can enhance test performance. Motivated by this observation, we systematically examine the effect of varying the residual coefficient  $\beta$ on the convergence rate. As depicted in \Cref{diff_beta_L1}, we observe that increasing the value of $\beta$ accelerates convergence.
Although the convergence rate for $\beta = 0.5$ or $0.75$ is slower than that of $\beta = 1$, it remains significantly faster than the model without residual connections, reinforcing the critical role of residual coefficients in efficient learning.
Additionally, for $\beta \in \{0, 0.25, 0.5, 0.75, 1\}$, the values of
$\frac{\min_{p}\sigma_{\min}^2(\mZ^{(0)}(\mX_p))}{\max_{p}\|\mZ^{(0)}(\mX_p) \|}  $ are given by $\{7.74\times 10^{-14},  0.16,0.32, 0.48, 0.56 \}$, highlighting the substantial influence of residual connections on convergence speed.

In the second experiment, we further empirically validate the effect of residual connections by conducting sentiment classification on the SST-2 dataset \citep{socher2013recursive}. All models are initialized from GPT-2 (small) \citep{radford2019language}, which consists of 12 Transformer layers with a hidden dimension of 768 and 12 attention heads. To isolate architectural effects, we truncate the pretrained model to its first $L$ layers and optionally remove residual connections. We evaluate models with and without residual connections for $L \in \{2, 6, 10\}$. Cross-entropy loss is used for training, and the detailed training procedure follows that of the third experiment in \cite{qin2026learning}. As shown in \Cref{diff_language_model}, Transformers equipped with residual connections consistently achieve lower training error than their counterparts without residual connections. Moreover, for models with residual connections, the training error further decreases as the number of layers $L$ increases.

\section{Conclusion}
\label{sec:conclusion}

This paper presented a comprehensive analysis of the convergence behavior of a single-layer Transformer architecture, incorporating a single-head self-attention mechanism, a feedforward neural network, and residual connections. Our theoretical investigation established that, under proper initialization, the gradient descent method exhibits a linear convergence rate. The convergence speed is primarily determined by the minimum and maximum singular values  of the output matrix produced by the attention layer, emphasizing the critical role of residual connections in maintaining numerical stability. Experimental results further substantiated our theoretical insights.

%{\small
%%\bibliographystyle{alpha}
%\bibliographystyle{unsrt}
%\bibliography{reference1}
%}

\newpage
\balance

\clearpage

\onecolumn

\appendices

\section{Gradient Derivations}
\label{sec:summary of gradients}

To simplify notation, we omit the explicit dependence of the loss function on all weight matrices and write $\nabla_{\mW_b}L$ instead of $\nabla_{\mW_b}L(\mW_1, \mW_2, \mW_Q, \mW_K, \mW_V, \mW_U)$ for $b \in \{1, 2, Q, K, V, U\}$. The gradients are summarized as follows:
\begin{eqnarray}
    \label{gradient of W1}
    \nabla_{\mW_1} L &\!\!\!\!=\!\!\!\!& \sum_{p=1}^P\mZ(\mX_p)^\top \bigg( \phi_r'(\mZ(\mX_p)\mW_1) \odot \bigg((F_{\mTheta}(\mX_p) - \mY_p)\mW_U^\top\mW_2^\top\bigg)\bigg),\\
    \label{gradient of W2}
    \nabla_{\mW_2} L &\!\!\!\!=\!\!\!\!& \sum_{p=1}^P(\phi_r(\mZ(\mX_p)\mW_1))^\top(F_{\mTheta}(\mX_p) - \mY_p)\mW_U^\top,\\
    \label{gradient of WU}
    \nabla_{\mW_U} L &\!\!\!\!=\!\!\!\!& \sum_{p=1}^P(\phi_r(\mZ(\mX_p)\mW_1)\mW_2 +\mZ(\mX_p))^\top(F_{\mTheta}(\mX_p) - \mY_p),\\
    \label{gradient of WV}
    \nabla_{\mW_V} L &\!\!\!\!=\!\!\!\!&\sum_{p=1}^P\mX_p^\top\phi_s\bigg(\frac{\mX_p\mW_Q\mW_K^\top \mX_p^\top}{\sqrt{d_{QK}}}\bigg)^\top \bigg( \phi_r'(\mZ(\mX_p)\mW_1) \odot \bigg((F_{\mTheta}(\mX_p) - \mY_p)\mW_U^\top\mW_2^\top\bigg)\bigg)\mW_1^\top\nonumber\\
    &\!\!\!\!\!\!\!\!&   + \sum_{p=1}^P\mX_p^\top\phi_s\bigg(\frac{\mX_p\mW_Q\mW_K^\top \mX_p^\top}{\sqrt{d_{QK}}}\bigg)^\top(F_{\mTheta}(\mX_p) - \mY_p)\mW_U^\top,\\
    \label{gradient of WQ}
    \nabla_{\mW_Q} L  &\!\!\!\!=\!\!\!\!& \sum_{p=1}^P\sum_{i=1}^{M}\mX_p^\top(i,:)(F_{\mTheta}(\mX_p)(i,:) - \mY_p(i,:))\mW_U^\top\mW_V^\top\mX_p^\top\phi_s'\bigg(\frac{\mX_p(i,:)\mW_Q\mW_K^\top \mX_p^\top}{\sqrt{d_{QK}}}\bigg)\mX_p\mW_K \nonumber\\
    &\!\!\!\!\!\!\!\!&  + \sum_{p=1}^P\sum_{i=1}^{M}\mX_p^\top(i,:)\bigg( \phi_r'(\mZ(\mX_p)(i,:)\mW_1) \odot \bigg((F_{\mTheta}(\mX_p)(i,:) - \mY_p(i,:))\mW_U^\top\mW_2^\top\bigg)\bigg)\nonumber\\
    &\!\!\!\!\!\!\!\!&\cdot\mW_1^\top\mW_V^\top\mX_p^\top\phi_s'\bigg(\frac{\mX_p(i,:)\mW_Q\mW_K^\top \mX_p^\top}{\sqrt{d_{QK}}}\bigg)\mX_p\mW_K,\\
    \label{gradient of WK}
    \nabla_{\mW_K} L  &\!\!\!\!=\!\!\!\!& \sum_{p=1}^P\sum_{i=1}^{M}\mX_p^\top\bigg((F_{\mTheta}(\mX_p)(i,:) - \mY_p(i,:))\mW_U^\top\mW_V^\top\mX_p^\top \phi_s'\bigg(\frac{\mX_p\mW_Q\mW_K^\top \mX_p^\top}{\sqrt{d_{QK}}}\bigg)\bigg)^\top\mX_p(i,:)\mW_Q \nonumber\\
    &\!\!\!\!\!\!\!\!&  +\sum_{p=1}^P\sum_{i=1}^{M}\mX_p^\top\bigg(\bigg( \phi_r'(\mZ(\mX_p)(i,:)\mW_1) \odot \bigg((F_{\mTheta}(\mX_p)(i,:) - \mY_p(i,:))\mW_U^\top\mW_2^\top\bigg)\bigg)\nonumber\\
    &\!\!\!\!\!\!\!\!& \cdot \mW_1^\top\mW_V^\top\mX_p^\top \phi_s'\bigg(\frac{\mX_p(i,:)\mW_Q\mW_K^\top \mX_p^\top}{\sqrt{d_{QK}}}\bigg)\bigg)^\top\mX_p(i,:)\mW_Q.
\end{eqnarray}
Note that to derive the last two gradients, we need to rewrite the loss function as $L(\mW_1, \mW_2,\mW_Q, \mW_K, \mW_V, \mW_U)= \frac{1}{2}\sum_{p=1}^P\sum_{i=1}^{M}\|(\phi_r((\phi_s(\frac{\mX_p(i,:)\mW_Q\mW_K^\top \mX_p^\top}{\sqrt{d_{QK}}})\mX_p\mW_V + \mX_p(i,:))\mW_1)\mW_2 +\phi_s(\frac{\mX_p(i,:)\mW_Q\mW_K^\top \mX_p^\top}{\sqrt{d_{QK}}})\mX_p\mW_V  + \mX_p(i,:)) \mW_U  - \mY_p(i,:) \|_2^2$. Here, the softmax function $\phi_s: \R^{1 \times M} \rightarrow \R^{1 \times M}$ is defined as
\begin{eqnarray}
    \label{definition of softmax function}
    \phi_s(\va) = \begin{bmatrix} \frac{e^{\va(1,:)}}{\sum_{i=1}^M e^{\va(i,:)}} \ \cdots \ \frac{e^{\va(M,:)}}{\sum_{i=1}^M e^{\va(i,:)}} \end{bmatrix}.
\end{eqnarray}
Furthermore, we define the Jacobian of the softmax function, evaluated at
\begin{eqnarray}
    \label{each row in the softmax derivative}
    \phi_s'\bigg(\frac{\mX_p(i,:)\mW_Q\mW_K^\top \mX_p^\top}{\sqrt{d_{QK}}}\bigg) = \text{diag}(\vs) - \vs^\top\vs,
\end{eqnarray}
where $\vs = \phi_s(\frac{\mX_p(i,:)\mW_Q\mW_K^\top \mX_p^\top}{\sqrt{d_{QK}}})$ is the softmax output vector.

\section{Detailed version of \Cref{Theorem of convergence analysis of transformer model simplified}}
\label{detailed version of convergence rate in the tranformer model simplified}

Before proceeding to the proof of \Cref{Theorem of convergence analysis of transformer model simplified}, we first present and prove the following auxiliary theorem.
\begin{theorem}
\label{Theorem of convergence analysis of transformer model detailed}
Given a dataset $\{\mX_p, \mY_p\}_{p=1}^{P}$, we consider a single-layer Transformer architecture equipped with a single-head self-attention mechanism, a feedforward neural network and residual connections. The goal is to model or approximate the underlying distribution of the data. Define $\alpha = \frac{\sigma_{\min}^2(\mW_U^{(0)})\sigma_{\min}^2(\vphi_P^{(0)})}{16}$. When the initialization satisfies
\begin{eqnarray}
    \label{requirement of the initialization transformer theorem  detailed}
    \Phi^{\frac{1}{2}}(\vtheta^{(0)}) &\!\!\!\!\leq \!\!\!\!& \min\Bigg\{ \frac{2\sqrt{2}\alpha (1 - (1 - \mu \alpha)^{\frac{1}{2}})}{27\sqrt{P}\max_{p}\|\mZ^{(0)}(\mX_p)\|}\cdot\min\bigg\{  \frac{\sigma_{\min}(\mW_1^{(0)})}{ \|\mW_2^{(0)}\| \|\mW_U^{(0)}\|}, \frac{\sigma_{\min}(\mW_2^{(0)})}{ \|\mW_1^{(0)}\| \|\mW_U^{(0)}\|} \bigg\},\nonumber\\
    &\!\!\!\!\!\!\!\!& \frac{2\sqrt{2}\alpha (1 - (1 - \mu \alpha)^{\frac{1}{2}})}{27\sqrt{P}(1 + \|\mW_1^{(0)} \| \|\mW_2^{(0)}\|)}\cdot\min\bigg\{ \frac{\sigma_{\min}(\mW_U^{(0)})}{\max_{p}\|\mZ^{(0)}(\mX_p)\|}, \frac{\sigma_{\min}(\mW_V^{(0)})}{ \sqrt{M}\max_{p}\|\mX_p\|\|\mW_U^{(0)}\|}\bigg\},\nonumber\\
    &\!\!\!\!\!\!\!\!& \frac{4\sqrt{2}\alpha (1 - (1 - \mu \alpha)^{\frac{1}{2}})}{243\sqrt{MP}\max_{i,p}\|\mX_p(i,:)\|_2\|\mX_p\|^2(1 + \|\mW_1^{(0)}\|\|\mW_2^{(0)}\|)\|\mW_U^{(0)}\|\|\mW_V^{(0)}\|}\min\bigg\{ \frac{\sigma_{\min}(\mW_Q^{(0)})}{\|\mW_K^{(0)}\|}, \frac{\sigma_{\min}(\mW_K^{(0)})}{\|\|\mW_Q^{(0)}\|}\bigg\},\nonumber\\
    &\!\!\!\!\!\!\!\!& \frac{\min_{p}\sigma_{\min}(\mZ^{(0)}(\mX_p))}{2C_1},2\sqrt{C_2},  \frac{\min\{\min_{p}\sigma_{\min}(\phi_r(\mZ^{(0)} \mW_1^{(0)})),\sigma_{\min}(\vphi_P^{(0)}) \} }{3C_1P\|\mW_1^{(0)}\| + \frac{27\sqrt{2}P^{\frac{3}{2}}}{4\alpha (1 - (1 - \mu \alpha)^{\frac{1}{2}})}\|\mZ^{(0)}\|^2 \|\mW_2^{(0)}\| \|\mW_U^{(0)}\|} \Bigg\},
\end{eqnarray}
where $  C_1 = \frac{2187M\sqrt{2P}\max_{p}\|\mX_p\|^3(\max_{i,p}\|\mX_p(i,:)\|_2\|\mX_p\|^2)\|\mW_U^{(0)}\|\|\mW_V^{(0)}\|^2(\|\mW_Q^{(0)}\|^2 + \|\mW_K^{(0)}\|^2 )}{32\sqrt{d_{QK}}\alpha (1 - (1 - \mu \alpha)^{\frac{1}{2}})}(1 + \|\mW_1^{(0)}\|\|\mW_2^{(0)}\|) \\  +\frac{27\sqrt{2}M\sqrt{P}(\max_{p}\|\mX_p\|)^2\|\mW_U^{(0)}\|}{8\alpha (1 - (1 - \mu \alpha)^{\frac{1}{2}})}(1  + \|\mW_1^{(0)} \| \|\mW_2^{(0)}\|) $ and $ C_2  =  \frac{6561\|\mW_2^{(0)}\|^2\|\mW_U^{(0)}\|^2\max_{p}\|\mZ^{(0)}(\mX_p)\|^2\sigma_{\min}^2(\mW_1^{(0)})}{64} + \\ \frac{81\max_{p}\|\mZ^{(0)}(\mX_p)\|^2\sigma_{\min}^2(\mW_U^{(0)})}{4}  +  \big(9\max_{p}\|\mZ^{(0)}(\mX_p)\|^2 + \frac{729\|\mW_1^{(0)}\|^2\|\mW_2^{(0)}\|^2\|\mW_U^{(0)}\|^2}{16}  \big)\big(\frac{27 M\sqrt{P}(\max_{p}\|\mX_p\|)^2\sigma_{\min}^2(\mW_V^{(0)})}{4} + \\ \frac{2187M\sqrt{P}\max_{p}\|\mX_p\|^3(\max_{i,p}\|\mX_p(i,:)\|_2\|\mX_p\|^2)\|\mW_V^{(0)}\|^2\|\mW_Q^{(0)}\|^2\sigma_{\min}^2(\mW_K^{(0)})}{16d_{QK}} +  \\ \frac{2187M\sqrt{P}\max_{p}\|\mX_p\|^3(\max_{i,p}\|\mX_p(i,:)\|_2\|\mX_p\|^2)\|\mW_V^{(0)}\|^2\|\mW_K^{(0)}\|^2\sigma_{\min}^2(\mW_Q^{(0)})}{16d_{QK}} \big) $.
Using the gradient descent in \eqref{every gradient descent}, we have
\begin{eqnarray}
    \label{convergence analysis of GD transformer in the theorem detailed}
    \Phi(\vtheta^{(t+1)}) \leq   (1 - \mu \alpha)\Phi(\vtheta^{(t)}),
\end{eqnarray}
where the learning rate satisfies $\mu \leq \min\{ \frac{1}{C}, \frac{1}{\alpha}\}$, and we define the following two constants  $    C^2  =  \frac{2187 PC_F}{32} \max_{p}\|\mZ^{(0)}(\mX_p)\|^2  \cdot\|\mW_U^{(0)} \|^2 (\|\mW_1^{(0)} \|^2 + \|\mW_2^{(0)} \|^2) + PC_F \max_{p}\|\mZ^{(0)}(\mX_p)\|^2 \big(\frac{2187}{16}\|\mW_1^{(0)}\|^2\|\mW_2^{(0)}\|^2 + 27 \big)  + \frac{PC_F}{d_{QK}} (\max_{p}\|\mX_p\|)^6 \|\mW_U^{(0)}\|^2 \\  \cdot\|\mW_Q^{(0)}\|^2  \|\mW_K^{(0)}\|^2\big(\frac{2187}{16} + \frac{177147}{256} \|\mW_1^{(0)}\|^2\|\mW_2^{(0)}\|^2\big)  + \frac{2187PC_F M}{4}\max_{i,p}\|\mX_p(i,:)\|_2^2(\max_p\|\mX_p\|)^4 \|\mW_U^{(0)}\|^2\|\mW_V^{(0)}\|^2\\ \cdot(\|\mW_K^{(0)}\|^2   + \|\mW_Q^{(0)}\|^2 )  + \frac{177147PC_F M}{64}\max_{i,p}\|\mX_p(i,:)\|_2^2(\max_p\|\mX_p\|)^4 \|\mW_1^{(0)}\|^2   \|\mW_2^{(0)}\|^2\\ \cdot\|\mW_U^{(0)}\|^2\|\mW_V^{(0)}\|^2 (\|\mW_K^{(0)}\|^2 + \|\mW_Q^{(0)}\|^2 )$
and $    C_F = \max \big\{
\frac{729\|\mW_2^{(0)}\|^2\|\mW_U^{(0)}\|^2\max_{p}\|\mZ^{(0)}(\mX_p)\|^2}{16},  \\   \big(\frac{729\|\mW_1^{(0)}\|^2\|\mW_2^{(0)}\|^2\|\mW_U^{(0)}\|^2}{16} + 9\max_{p}\|\mZ^{(0)}(\mX_p)\|^2 \big)\frac{243M\sqrt{P}\max_{p}\|\mX_p\|^3(\max_{i,p}\|\mX_p(i,:)\|_2\|\mX_p\|^2)\|\mW_V^{(0)}\|^2\|\mW_Q^{(0)}\|^2}{4d_{QK}}, \\   \big(\frac{729\|\mW_1^{(0)}\|^2\|\mW_2^{(0)}\|^2\|\mW_U^{(0)}\|^2}{16} + 9\max_{p}\|\mZ^{(0)}(\mX_p)\|^2 \big) \frac{243M\sqrt{P}\max_{p}\|\mX_p\|^3(\max_{i,p}\|\mX_p(i,:)\|_2\|\mX_p\|^2)\|\mW_V^{(0)}\|^2\|\mW_K^{(0)}\|^2}{4d_{QK}}, \\
27 M\sqrt{P}(\max_{p}\|\mX_p\|)^2\big( \max_{p}\|\mZ^{(0)}(\mX_p)\|^2 + \frac{81\|\mW_1^{(0)}\|^2\|\mW_2^{(0)}\|^2\|\mW_U^{(0)}\|^2}{16}  \big), 9\max_{p}\|\mZ^{(0)}(\mX_p)\|^2 \big\}$.

\end{theorem}

\begin{proof} We show by induction that, for every $t \geq 0$, the following holds:
\begin{eqnarray}
    \label{Summary of Proof requirements transformer}
    \begin{cases}
    \|\mW_{b,a}^{(t)}\|\leq \frac{3}{2}\|\mW_b^{(0)}\|, b = 1,2,U,V,Q,K, \\
    \sigma_{\min}(\mW_{b,a}^{(t)}) \geq \frac{\sigma_{\min}(\mW_b^{(0)})}{2}, b = 1,2,U,V,Q,K, \\
    \|\mZ_a^{(t)}(\mX_p)\|\leq \frac{3}{2}\|\mZ^{(0)}(\mX_p)\|, p=1,\dots, P,\\
    \sigma_{\min}(\mZ_a^{(t)}(\mX_p)) \geq \frac{\sigma_{\min}(\mZ^{(0)})(\mX_p)}{2}, p=1,\dots, P,  \\
    \sigma_{\min}(\phi_r(\mZ^{(t)}(\mX_p) \mW_1^{(t)})) \geq \frac{\sigma_{\min}(\phi_r(\mZ^{(0)}(\mX_p) \mW_1^{(0)}))}{2}, p=1,\dots, P,  \\
    \sigma_{\min}(\vphi_P^{(t)}) \geq \frac{\sigma_{\min}(\vphi_P^{(0)})}{2}, \\
    \Phi(\vtheta^{(t+1)})\leq (1 - \mu \alpha)\Phi(\vtheta^{(t)}).
    \end{cases}
\end{eqnarray}
where $\mW_{b,a}^{(t)} = \mW_b^{(t)} + a(\mW_b^{(t+1)} - \mW_b^{(t)})$ and $\mZ_a^{(t)}(\mX_p) = \mZ^{(t)}(\mX_p) + a(\mZ^{(t+1)}(\mX_p) - \mZ^{(t)}(\mX_p))$, $a\in[0,1]$.

\paragraph{Step I: Proof of first six inequalities in \eqref{Summary of Proof requirements transformer}. }

\textbf{Part I} We begin by bounding $\|\mW_{1,a}^{(t)}\|$ and $\sigma_{\min}(\mW_{1,a}^{(t)})$ where $\mW_{1,a}^{(t)} = \mW_1^{(t)} + a(\mW_1^{(t+1)} - \mW_1^{(t)})$ with $a\in[0,1]$. Specifically, we first expand $\|\mW_{1,a}^{(t)} - \mW_1^{(0)} \|_F$ as follows:
\begin{eqnarray}
    \label{upper bound W1 transformer}
    &\!\!\!\!\!\!\!\!&\|\mW_{1,a}^{(t)} - \mW_1^{(0)} \|_F\nonumber\\
    &\!\!\!\!\leq\!\!\!\!& \mu \sum_{s=0}^{t}\|\nabla_{\mW_1} L(\mW_1^{(s)}, \mW_2^{(s)},\mW_Q^{(s)}, \mW_K^{(s)}, \mW_V^{(s)}, \mW_U^{(s)}) \|_F\nonumber\\
    &\!\!\!\!\leq\!\!\!\!& \mu \sum_{s=0}^{t}\sum_{p=1}^{P} \|\mZ^{(s)}(\mX_p)\| \|( \phi_r'(\mZ^{(s)}(\mX_p)\mW_1^{(s)}) \odot ((F^{(s)}_{\mTheta}(\mX_p) - \mY_p){\mW_U^{(s)}}^\top{\mW_2^{(s)}}^\top))\|_F\nonumber\\
    &\!\!\!\!\leq\!\!\!\!& \mu \sum_{s=0}^{t}\sum_{p=1}^{P} \|\mZ^{(s)}(\mX_p)\| \|\mW_2^{(s)}\| \|\mW_U^{(s)}\| \|F^{(s)}_{\mTheta}(\mX_p) - \mY\|_F\nonumber\\
    &\!\!\!\!\leq\!\!\!\!& \frac{27\sqrt{2}\mu\sqrt{P}}{8}\max_p\|\mZ^{(0)}(\mX_p)\| \|\mW_2^{(0)}\| \|\mW_U^{(0)}\| \sum_{s=0}^{t}(1 - \mu \alpha)^{\frac{s}{2}} \Phi^{\frac{1}{2}}(\vtheta^{(0)})\nonumber\\
    &\!\!\!\! = \!\!\!\!&\frac{27\sqrt{2}\mu\sqrt{P}}{8}\max_p\|\mZ^{(0)}(\mX_p)\| \|\mW_2^{(0)}\| \|\mW_U^{(0)}\| \Phi^{\frac{1}{2}}(\vtheta^{(0)}) \frac{1 - (1 - \mu \alpha)^{\frac{t+1}{2}}}{1 - (1 - \mu \alpha)^{\frac{1}{2}}}\nonumber\\
    &\!\!\!\!\leq\!\!\!\!& \frac{27\sqrt{2}\sqrt{P}}{8\alpha (1 - (1 - \mu \alpha)^{\frac{1}{2}})}\max_p\|\mZ^{(0)}(\mX_p)\| \|\mW_2^{(0)}\| \|\mW_U^{(0)}\| \Phi^{\frac{1}{2}}(\vtheta^{(0)}).
\end{eqnarray}
where $\mu \leq \frac{1}{\alpha}$. To further guarantee $\|\mW_{1,a}^{(t)} - \mW_1^{(0)} \|_F \leq \frac{\sigma_{\min}(\mW_1^{(0)})}{2} $, the initialization requirement should satisfy
\begin{eqnarray}
    \label{requirement of initialization1 transformer}
    \Phi^{\frac{1}{2}}(\vtheta^{(0)}) \leq \frac{2\sqrt{2}\alpha (1 - (1 - \mu \alpha)^{\frac{1}{2}})}{27\sqrt{P}}\frac{\sigma_{\min}(\mW_1^{(0)})}{\max_p\|\mZ^{(0)}(\mX_p)\| \|\mW_2^{(0)}\| \|\mW_U^{(0)}\|}.
\end{eqnarray}
Based on $\sigma_{\min}(\mW_1^{(0)}) - \sigma_{\min}(\mW_{1,a}^{(t)}) \leq \|\mW_{1,a}^{(t)} - \mW_1^{(0)} \| \leq  \frac{\sigma_{\min}(\mW_1^{(0)})}{2}$, we can obtain
\begin{eqnarray}
    \label{bound of W1 transformer}
    \frac{\sigma_{\min}(\mW_1^{(0)})}{2} \leq \sigma_{\min}(\mW_{1,a}^{(t)}) \leq \|\mW_{1,a}^{(t)}\| \leq \frac{3\|\mW_1^{(0)}\|}{2}.
\end{eqnarray}

Similarly, we respectively expand other terms $\|\mW_{b,a}^{(t)} - \mW_b^{(0)} \|_F, b = 2,U,V,Q,K,$ as follows:
\begin{eqnarray}
    \label{upper bound W2 transformer}
    &\!\!\!\!\!\!\!\!&\|\mW_{2,a}^{(t)} - \mW_2^{(0)} \|_F\nonumber\\
    &\!\!\!\!\leq\!\!\!\!&  \mu \sum_{s=0}^{t}\sum_{p=1}^{P} \|(\mZ^{(s)}(\mX_p)\mW_1^{(s)})^\top(F^{(s)}_{\mTheta}(\mX_p) - \mY_p){\mW_U^{(s)}}^\top\|_F\nonumber\\
    &\!\!\!\!\leq\!\!\!\!& \mu \sum_{s=0}^{t}\sum_{p=1}^{P} \|\mZ^{(s)}(\mX_p)\| \|\mW_1^{(s)}\| \|\mW_U^{(s)}\| \|F^{(s)}_{\mTheta}(\mX_p) - \mY_p\|_F\nonumber\\
    &\!\!\!\!\leq\!\!\!\!& \frac{27\sqrt{2}\sqrt{P}}{8\alpha (1 - (1 - \mu \alpha)^{\frac{1}{2}})}\max_{p}\|\mZ^{(0)}(\mX_p)\| \|\mW_1^{(0)}\| \|\mW_U^{(0)}\| \Phi^{\frac{1}{2}}(\vtheta^{(0)}),\\
    \label{upper bound WU transformer}
    &\!\!\!\!\!\!\!\!&\|\mW_{U,a}^{(t)} - \mW_U^{(0)} \|_F\nonumber\\
    &\!\!\!\!\leq\!\!\!\!&  \mu \sum_{s=0}^{t}\sum_{p=1}^{P} \|(\phi_r(\mZ^{(s)}(\mX_p)\mW_1^{(s)})\mW_2^{(s)})^\top(F^{(s)}_{\mTheta}(\mX_p) - \mY_p)\|_F +\|(\mZ^{(s)}(\mX_p))^\top(F^{(s)}_{\mTheta}(\mX_p) - \mY_p)\|_F\nonumber\\
    &\!\!\!\!\leq\!\!\!\!& \mu \sum_{s=0}^{t}\sum_{p=1}^{P} \|\mW_2^{(s)}\|\|(\mZ^{(s)}(\mX_p)\mW_1^{(s)})^\top(F^{(s)}_{\mTheta}(\mX_p) - \mY_p)\|_F +\|\mZ^{(s)}(\mX_p)\|\|F^{(s)}_{\mTheta}(\mX_p) - \mY_p\|_F\nonumber\\
    &\!\!\!\!\leq\!\!\!\!& \mu \sum_{s=0}^{t}\sum_{p=1}^{P} \|\mW_2^{(s)}\| \|\mZ^{(s)}(\mX_p)\| \|\mW_1^{(s)}\|\|F^{(s)}_{\mTheta}(\mX_p) - \mY_p\|_F +\|\mZ^{(s)}(\mX_p)\|\|F^{(s)}_{\mTheta}(\mX_p) - \mY_p\|_F\nonumber\\
    &\!\!\!\!\leq\!\!\!\!& \frac{27\sqrt{2}\sqrt{P}}{8\alpha (1 - (1 - \mu \alpha)^{\frac{1}{2}})}(1 + \|\mW_1^{(0)}\|\|\mW_2^{(0)}\| )\max_{p}\|\mZ^{(0)}(\mX_p)\| \Phi^{\frac{1}{2}}(\vtheta^{(0)}),
        \end{eqnarray}
    \begin{eqnarray}
    \label{upper bound WV transformer}
    &\!\!\!\!\!\!\!\!&\|\mW_{V,a}^{(t)} - \mW_V^{(0)} \|_F\nonumber\\
    &\!\!\!\!\leq\!\!\!\!&  \mu \sum_{s=0}^{t} \sum_{p=1}^{P}\|\mX_p^\top\phi_s\bigg(\frac{\mX_p\mW_Q^{(s)}{\mW_K^{(s)}}^\top \mX_p^\top}{\sqrt{d_{QK}}}\bigg)^\top \bigg( \phi_r'(\mZ^{(s)}(\mX_p)\mW_1^{(s)}) \odot \bigg((F^{(s)}_{\mTheta}(\mX_p)- \mY_p){\mW_U^{(s)}}^\top{\mW_2^{(s)}}^\top\bigg)\bigg){\mW_1^{(s)}}^\top\|_F\nonumber\\
    &\!\!\!\!\!\!\!\!&   + \|\mX_p^\top\phi_s\bigg(\frac{\mX_p\mW_Q^{(s)}{\mW_K^{(s)}}^\top \mX_p^\top}{\sqrt{d_{QK}}}\bigg)^\top(F^{(s)}_{\mTheta}(\mX_p) - \mY_p){\mW_U^{(s)}}^\top\|_F\nonumber\\
    &\!\!\!\!\leq\!\!\!\!& \mu \sum_{s=0}^{t} \sum_{p=1}^{P}\sqrt{M}\|\mX_p\|\|\mW_1^{(s)} \| \|(F^{(s)}_{\mTheta}(\mX_p) - \mY_p){\mW_U^{(s)}}^\top{\mW_2^{(s)}}^\top\|_F  +  \sqrt{M}\|\mX_p\|\|(F^{(s)}_{\mTheta}(\mX_p) - \mY_p){\mW_U^{(s)}}^\top\|_F \nonumber\\
    &\!\!\!\!\leq\!\!\!\!& \mu\sqrt{M} \sum_{s=0}^{t}\sum_{p=1}^{P}\|\mX_p\|\|\mW_1^{(s)} \| \|\mW_2^{(s)}\| \|\mW_U^{(s)}\| \|F^{(s)}_{\mTheta}(\mX_p) - \mY_p\|_F + \|\mX_p\|\|\mW_U^{(s)}\|\|F^{(s)}_{\mTheta}(\mX_p) - \mY_p\|_F \nonumber\\
    &\!\!\!\!\leq\!\!\!\!& \frac{27\sqrt{2}\sqrt{MP}\max_{p}\|\mX_p\|}{8\alpha (1 - (1 - \mu \alpha)^{\frac{1}{2}})}(1 + \|\mW_1^{(0)} \| \|\mW_2^{(0)}\|) \|\mW_U^{(0)}\| \Phi^{\frac{1}{2}}(\vtheta^{(0)}),
\end{eqnarray}
where the second inequality follows \Cref{property of softmax function and its derivative}.

\begin{eqnarray}
    \label{upper bound WQ transformer}
    &\!\!\!\!\!\!\!\!&\|\mW_{Q,a}^{(t)} - \mW_Q^{(0)} \|_F\nonumber\\
    &\!\!\!\!\leq\!\!\!\!&  \mu \sum_{s=0}^{t}\sum_{p=1}^{P} \sum_{i=1}^{M}\bigg(2\|\mX_p(i,:)\|_2\|F^{(s)}_{\mTheta}(\mX_p)(i,:) - \mY_p(i,:)\|_2\|\mW_U^{(s)}\| \|\mW_V^{(s)}\|\|\mX_p\|^2\|\mW_K^{(s)}\|\nonumber\\
    &\!\!\!\!\!\!\!\!& + 2\|\mX_p(i,:)\|_2\|F^{(s)}_{\mTheta}(\mX_p)(i,:) - \mY_p(i,:)\|_2 \|\mW_U^{(s)}\| \|\mW_2^{(s)}\| \|\mW_1^{(s)}\| \|\mW_V^{(s)}\|\|\mX_p\|^2\|\mW_K^{(s)}\|\bigg)\nonumber\\
    &\!\!\!\!\leq\!\!\!\!& \mu \sum_{s=0}^{t}\sum_{p=1}^{P} \frac{243\sqrt{M}}{16} \max_{i}\|\mX_p(i,:)\|_2\|\mX_p\|^2 (1 + \|\mW_1^{(0)}\|\|\mW_2^{(0)}\|)\|\mW_U^{(0)}\|\|\mW_V^{(0)}\|\|\mW_K^{(0)}\|\|F^{(s)}_{\mTheta}(\mX_p) - \mY_p\|_F\nonumber\\
    &\!\!\!\!\leq\!\!\!\!& \frac{243\sqrt{2MP}\max_{i,p}\|\mX_p(i,:)\|_2\|\mX_p\|^2}{16\alpha (1 - (1 - \mu \alpha)^{\frac{1}{2}})}(1 + \|\mW_1^{(0)}\|\|\mW_2^{(0)}\|)\|\mW_U^{(0)}\|\|\mW_V^{(0)}\|\|\mW_K^{(0)}\|\Phi^{\frac{1}{2}}(\vtheta^{(0)}),\\
    \label{upper bound WK transformer}
    &\!\!\!\!\!\!\!\!&\|\mW_{K,a}^{(t)} - \mW_K^{(0)} \|_F\nonumber\\
    &\!\!\!\!\leq\!\!\!\!&  \mu \sum_{s=0}^{t}\sum_{p=1}^{P} \sum_{i=1}^{M}\bigg( 2\|\mX_p(i,:)\|_2\|F^{(s)}_{\mTheta}(\mX_p)(i,:) - \mY_p(i,:)\|_2\|\mW_U^{(s)}\| \|\mW_V^{(s)}\|\|\mX_p\|^2\|\mW_Q^{(s)}\|\nonumber\\
    &\!\!\!\!\!\!\!\!& + 2\|\mX_p(i,:)\|_2\|F^{(s)}_{\mTheta}(\mX_p)(i,:) - \mY_p(i,:)\|_2 \|\mW_U^{(s)}\| \|\mW_2^{(s)}\| \|\mW_1^{(s)}\| \|\mW_V^{(s)}\|\|\mX_p\|^2\|\mW_Q^{(s)}\|\bigg)\nonumber\\
    &\!\!\!\!\leq\!\!\!\!& \mu \sum_{s=0}^{t}\sum_{p=1}^{P} \frac{243\sqrt{M}}{16} \max_{i}\|\mX_p(i,:)\|_2\|\mX_p\|^2 (1 + \|\mW_1^{(0)}\|\|\mW_2^{(0)}\|)\|\mW_U^{(0)}\|\|\mW_V^{(0)}\|\|\mW_Q^{(0)}\|\|F^{(s)}_{\mTheta}(\mX_p) - \mY_p\|_F\nonumber\\
    &\!\!\!\!\leq\!\!\!\!& \frac{243\sqrt{2MP}\max_{i,p}\|\mX_p(i,:)\|_2\|\mX_p\|^2}{16\alpha (1 - (1 - \mu \alpha)^{\frac{1}{2}})}(1 + \|\mW_1^{(0)}\|\|\mW_2^{(0)}\|)\|\mW_U^{(0)}\|\|\mW_V^{(0)}\|\|\mW_Q^{(0)}\|\Phi^{\frac{1}{2}}(\vtheta^{(0)}).
\end{eqnarray}

When the initialization requirement satisfies
\begin{eqnarray}
    \label{the requirement of initialization}
    \Phi^{\frac{1}{2}}(\vtheta^{(0)})&\!\!\!\!\leq \!\!\!\!& \min\Bigg\{\frac{2\sqrt{2}\alpha (1 - (1 - \mu \alpha)^{\frac{1}{2}})}{27\sqrt{P}}\frac{\sigma_{\min}(\mW_2^{(0)})}{\max_{p}\|\mZ^{(0)}(\mX_p)\| \|\mW_1^{(0)}\| \|\mW_U^{(0)}\|},\nonumber\\
    &\!\!\!\!\!\!\!\!& \frac{2\sqrt{2}\alpha (1 - (1 - \mu \alpha)^{\frac{1}{2}})}{27\sqrt{P}(1 + \|\mW_1^{(0)} \| \|\mW_2^{(0)}\|)}\cdot\min\bigg\{ \frac{\sigma_{\min}(\mW_U^{(0)})}{\max_{p}\|\mZ^{(0)}(\mX_p)\|}, \frac{\sigma_{\min}(\mW_V^{(0)})}{ \sqrt{M}\max_{p}\|\mX_p\|\|\mW_U^{(0)}\|}\bigg\},\nonumber\\
    &\!\!\!\!\!\!\!\!& \frac{4\sqrt{2}\alpha (1 - (1 - \mu \alpha)^{\frac{1}{2}})}{243\sqrt{MP}\max_{i,p}\|\mX_p(i,:)\|_2\|\mX_p\|^2(1 + \|\mW_1^{(0)}\|\|\mW_2^{(0)}\|)\|\mW_U^{(0)}\|\|\mW_V^{(0)}\|}  \min\bigg\{ \frac{\sigma_{\min}(\mW_Q^{(0)})}{\|\mW_K^{(0)}\|}, \frac{\sigma_{\min}(\mW_K^{(0)})}{\|\|\mW_Q^{(0)}\|}\bigg\}\Bigg\} ,\nonumber
\end{eqnarray}
we can guarantee
\begin{eqnarray}
    \label{bound of W2,U,V,Q,K transformer}
    \frac{\sigma_{\min}(\mW_b^{(0)})}{2} \leq \sigma_{\min}(\mW_{b,a}^{(t)}) \leq \|\mW_{b,a}^{(t)}\| \leq \frac{3\|\mW_b^{(0)}\|}{2}, b=2,U,V,Q,K.
\end{eqnarray}

\textbf{Part II} Next, we will prove $\frac{\sigma_{\min}(\mZ^{(0)}(\mX_p))}{2} \leq \sigma_{\min}(\mZ_a^{(t)}(\mX_p)) \leq \|\mZ_a^{(t)}(\mX_p)\| \leq \frac{3\|\mZ^{(0)}(\mX_p)\|}{2}$. Specifically, we first expand
\begin{eqnarray}
    \label{expansion of Z}
    &\!\!\!\!\!\!\!\!&\|\mZ_a^{(t)}(\mX_p) - \mZ^{(0)}(\mX_p)\|_F\nonumber\\
    &\!\!\!\! = \!\!\!\!& \| \phi_s(\frac{\mX_p\mW_{Q,a}^{(t)}{\mW_{K,a}^{(t)}}^\top \mX_p^\top}{\sqrt{d_{QK}}})\mX_p\mW_{V,a}^{(t)} - \phi_s(\frac{\mX_p\mW_Q^{(0)}{\mW_K^{(0)}}^\top \mX_p^\top}{\sqrt{d_{QK}}})\mX_p\mW_V^{(0)}\|_F\nonumber\\
    &\!\!\!\! \leq \!\!\!\!& \|\phi_s(\frac{\mX_p\mW_{Q,a}^{(t)}{\mW_{K,a}^{(t)}}^\top \mX_p^\top}{\sqrt{d_{QK}}}) - \phi_s(\frac{\mX_p\mW_Q^{(0)}{\mW_K^{(0)}}^\top \mX_p^\top}{\sqrt{d_{QK}}}) \|_F\|\mX_p\| \|\mW_{V,a}^{(t)}\|\nonumber\\
    &\!\!\!\!\!\!\!\!& + \|\phi_s(\frac{\mX_p\mW_Q^{(0)}{\mW_K^{(0)}}^\top \mX_p^\top}{\sqrt{d_{QK}}})\|_F \|\mX_p\| \|\mW_{V,a}^{(t)} - \mW_V^{(0)} \|_F\nonumber\\
    &\!\!\!\! \leq \!\!\!\!& \frac{3\sqrt{M}\|\mX_p\|^3\|\mW_V^{(0)}\|}{\sqrt{d_{QK}}}\|\mW_{Q,a}^{(t)}{\mW_{K,a}^{(t)}}^\top - \mW_Q^{(0)}{\mW_K^{(0)}}^\top \|_F +\sqrt{M}\|\mX_p\|\|\mW_{V,a}^{(t)} - \mW_V^{(0)} \|_F\nonumber\\
    &\!\!\!\! \leq \!\!\!\!& \sqrt{M}\|\mX_p\|\|\mW_{V,a}^{(t)} - \mW_V^{(0)} \|_F + \frac{9\sqrt{M}\|\mX_p\|^3\|\mW_V^{(0)}\|\|\mW_Q^{(0)}\|}{2\sqrt{d_{QK}}}\|\mW_{K,a}^{(t)} - \mW_K^{(0)} \|_F\nonumber\\
    &\!\!\!\!\!\!\!\!& + \frac{9\sqrt{M}\|\mX_p\|^3\|\mW_V^{(0)}\|\|\mW_K^{(0)}\|}{2\sqrt{d_{QK}}}\|\mW_{Q,a}^{(t)} - \mW_Q^{(0)} \|_F\nonumber\\
    &\!\!\!\! \leq \!\!\!\!& \frac{2187\sqrt{2}M\sqrt{P}\max_{p}\|\mX_p\|^3(\max_{i,p}\|\mX_p(i,:)\|_2\|\mX_p\|^2)\|\mW_U^{(0)}\|\|\mW_V^{(0)}\|^2(\|\mW_Q^{(0)}\|^2 + \|\mW_K^{(0)}\|^2 )}{32\sqrt{d_{QK}}\alpha (1 - (1 - \mu \alpha)^{\frac{1}{2}})}(1\nonumber\\
    &\!\!\!\!\!\!\!\!& + \|\mW_1^{(0)}\|\|\mW_2^{(0)}\|)\Phi^{\frac{1}{2}}(\vtheta^{(0)}) +\frac{27\sqrt{2}M\sqrt{P}(\max_{p}\|\mX_p\|)^2\|\mW_U^{(0)}\|}{8\alpha (1 - (1 - \mu \alpha)^{\frac{1}{2}})}(1 + \|\mW_1^{(0)} \| \|\mW_2^{(0)}\|) \Phi^{\frac{1}{2}}(\vtheta^{(0)})\nonumber\\
    &\!\!\!\! := \!\!\!\!& C_1 \Phi^{\frac{1}{2}}(\vtheta^{(0)}),
\end{eqnarray}
where the second inequality follows \Cref{property of softmax function and its derivative} and the fourth inequality uses \eqref{upper bound WV transformer}, \eqref{upper bound WQ transformer} and \eqref{upper bound WK transformer}.

Similar with the analysis of \eqref{bound of W1 transformer}, when the initialization satisfies $\Phi^{\frac{1}{2}}(\vtheta^{(0)}) \leq \frac{\min_{p}\sigma_{\min}(\mZ^{(0)}(\mX_p))}{2C_1}$, we can derive
\begin{eqnarray}
    \label{bound of Z transformer}
    \frac{\sigma_{\min}(\mZ^{(0)}(\mX_p))}{2} \leq \sigma_{\min}(\mZ_a^{(t)}(\mX_p)) \leq \|\mZ_a^{(t)}(\mX_p)\| \leq \frac{3\|\mZ^{(0)}(\mX_p)\|}{2},
\end{eqnarray}
where $a \in[0,1]$.

\textbf{Part III} Finally, we derive a lower bound for $\sigma_{\min}(\phi_r(\mZ^{(t)}(\mX_p) \mW_1^{(t)}))$. Specifically, we have
\begin{eqnarray}
    \label{expansion of sigmar for ZW1}
    &\!\!\!\!\!\!\!\!&\|\phi_r(\mZ^{(t)}(\mX_p) \mW_1^{(t)}) - \phi_r(\mZ^{(0)}(\mX_p) \mW_1^{(0)}) \|_F\nonumber\\
    &\!\!\!\!\leq \!\!\!\!&\|\phi_r(\mZ^{(t)}(\mX_p) \mW_1^{(t)}) - \phi_r(\mZ^{(0)}(\mX_p) \mW_1^{(t)}) \|_F + \|\phi_r(\mZ^{(0)}(\mX_p) \mW_1^{(t)}) - \phi_r(\mZ^{(0)}(\mX_p) \mW_1^{(0)}) \|_F\nonumber\\
    &\!\!\!\!\leq \!\!\!\!&\|\mZ^{(t)}(\mX_p) \mW_1^{(t)} - \mZ^{(0)}(\mX_p) \mW_1^{(t)} \|_F + \|\mZ^{(0)}(\mX_p) \mW_1^{(t)} - \mZ^{(0)}(\mX_p) \mW_1^{(0)} \|_F\nonumber\\
    &\!\!\!\!\leq \!\!\!\!&\frac{3\|\mW_1^{(0)}\|}{2}\|\mZ^{(t)}(\mX_p)  - \mZ^{(0)}(\mX_p)  \|_F + \|\mZ^{(0)}(\mX_p)\|\|\mW_1^{(t)} -  \mW_1^{(0)} \|_F\nonumber\\
    &\!\!\!\! \leq \!\!\!\!&  \frac{3C_1\|\mW_1^{(0)}\|}{2}\Phi^{\frac{1}{2}}(\vtheta^{(0)}) + \frac{27\sqrt{2P}}{8\alpha (1 - (1 - \mu \alpha)^{\frac{1}{2}})}\max_{p}\|\mZ^{(0)}(\mX_p)\|^2 \|\mW_2^{(0)}\| \|\mW_U^{(0)}\| \Phi^{\frac{1}{2}}(\vtheta^{(0)}),
\end{eqnarray}
where the last line uses \eqref{upper bound W1 transformer} and \eqref{expansion of Z}. When $$\Phi^{\frac{1}{2}}(\vtheta^{(0)}) \leq \frac{\min_{p}\sigma_{\min}(\phi_r(\mZ^{(0)}(\mX_p) \mW_1^{(0)}))}{3C_1\|\mW_1^{(0)}\| + \frac{27\sqrt{2P}}{4\alpha (1 - (1 - \mu \alpha)^{\frac{1}{2}})}\max_{p}\|\mZ^{(0)}(\mX_p)\|^2 \|\mW_2^{(0)}\| \|\mW_U^{(0)}\|}$$ is satisfied, following the analysis of \eqref{bound of W1 transformer}, we have
\begin{eqnarray}
    \label{lower bound for sigmar for ZW1}
    \sigma_{\min}(\phi_r(\mZ^{(t)}(\mX_p) \mW_1^{(t)})) \geq \frac{\sigma_{\min}(\phi_r(\mZ^{(0)}(\mX_p) \mW_1^{(0)}))}{2}.
\end{eqnarray}

Similarly, we can derive the lower bound for $\sigma_{\min}(\vphi_P^{(t)})$. Since $\|\vphi_P^{(t)} - \vphi_P^{(0)}  \|_F\leq \sum_{p=1}^{P}\|\phi_r(\mZ^{(t)}(\mX_p) \mW_1^{(t)}) - \phi_r(\mZ^{(0)}(\mX_p) \mW_1^{(0)})\|_F$, when $$\Phi^{\frac{1}{2}}(\vtheta^{(0)}) \leq \frac{\sigma_{\min}(\vphi_P^{(0)})}{3C_1P\|\mW_1^{(0)}\| + \frac{27\sqrt{2}P^{\frac{3}{2}}}{4\alpha (1 - (1 - \mu \alpha)^{\frac{1}{2}})}\max_{p}\|\mZ^{(0)}(\mX_p)\|^2 \|\mW_2^{(0)}\| \|\mW_U^{(0)}\|}$$ is met, we have $\sigma_{\min}(\vphi_P^{(t)}) \geq \frac{\sigma_{\min}(\vphi_P^{(0)})}{2}$.

\paragraph{Step II: Establishing the Lipschitz continuity of the gradient.} Before proving $\Phi(\vtheta^{(t+1)})\leq (1 - \mu \alpha)\Phi(\vtheta^{(t)})$, it is necessary to first establish the Lipschitz continuity of the gradient, i.e., $\|\Phi(\vtheta^{(t)}_a ) - \Phi(\vtheta^{(t)}) \|_2\leq C \|\vtheta^{(t)}_a  - \vtheta^{(t)}\|_2$ with $\vtheta^{(t)}_a = \vtheta^{(t)} + a(\vtheta^{(t+1)} - \vtheta^{(t)})$ and $a\in[0,1]$.
To proceed, we recall the result in \eqref{expansion of Z}, which states that
\begin{eqnarray}
    \label{expansion of Z at time t}
    &\!\!\!\!\!\!\!\!&\|\mZ_a^{(t)}(\mX_p) - \mZ^{(t)}(\mX_p)\|_F^2\nonumber\\
    &\!\!\!\! \leq \!\!\!\!& \frac{243M\sqrt{P}\|\mX_p\|^3(\max_{i,p}\|\mX_p(i,:)\|_2\|\mX_p\|^2)\|\mW_V^{(0)}\|^2\|\mW_Q^{(0)}\|^2}{4d_{QK}}\|\mW_{K,a}^{(t)} - \mW_K^{(t)} \|_F^2\nonumber\\
    &\!\!\!\!\!\!\!\!& + \frac{243M\sqrt{P}\|\mX_p\|^3(\max_{i,p}\|\mX_p(i,:)\|_2\|\mX_p\|^2)\|\mW_V^{(0)}\|^2\|\mW_K^{(0)}\|^2}{4d_{QK}}\|\mW_{Q,a}^{(t)} - \mW_Q^{(t)} \|_F^2\nonumber\\
    &\!\!\!\!\!\!\!\!& +3 M\sqrt{P}(\max_{p}\|\mX_p\|)^2\|\mW_{V,a}^{(t)} - \mW_V^{(t)} \|_F^2,
\end{eqnarray}
and further derive
\begin{eqnarray}
    \label{Lipschitz constant of F and Fa in the transofrmer}
    &\!\!\!\!\!\!\!\!&\|F_{\mTheta, a}^{(t)}(\mX_p) - F_{\mTheta}^{(t)}(\mX_p) \|_F^2\nonumber\\
    &\!\!\!\! = \!\!\!\!&\|(\phi_r(\mZ_a^{(t)}(\mX_p)\mW_{1,a}^{(t)})\mW_{2,a}^{(t)} +\mZ_a^{(t)}(\mX_p)) \mW_{U,a}^{(t)} - (\phi_r(\mZ^{(t)}(\mX_p)\mW_{1}^{(t)})\mW_{2}^{(t)} +\mZ^{(t)}(\mX_p)) \mW_{U}^{(t)}\|_F^2\nonumber\\
    &\!\!\!\! \leq \!\!\!\!& 2\|\phi_r(\mZ_a^{(t)}(\mX_p)\mW_{1,a}^{(t)})\mW_{2,a}^{(t)}\mW_{U,a}^{(t)} - \phi_r(\mZ^{(t)}(\mX_p)\mW_{1}^{(t)})\mW_{2}^{(t)}\mW_{U}^{(t)} \|_F^2\nonumber\\
    &\!\!\!\!\!\!\!\!&+ 2\|\mZ_a^{(t)}(\mX_p) \mW_{U,a}^{(t)} - \mZ^{(t)}(\mX_p) \mW_{U}^{(t)}\|_F^2\nonumber\\
    &\!\!\!\! \leq \!\!\!\!& \frac{81\|\mW_2^{(0)}\|^2\|\mW_U^{(0)}\|^2}{4} (\|\mZ_a^{(t)}(\mX_p)\mW_{1,a}^{(t)} - \mZ_a^{(t)}(\mX_p)\mW_{1}^{(t)} \|_F^2  + \|\mZ_a^{(t)}(\mX_p)\mW_{1}^{(t)} - \mZ^{(t)}(\mX_p)\mW_{1}^{(t)}\|_F^2)\nonumber\\
    &\!\!\!\!\!\!\!\!& + 4 (\|\mZ_a^{(t)}(\mX_p) \mW_{U,a}^{(t)} - \mZ_a^{(t)}(\mX_p) \mW_{U}^{(t)} \|_F^2 + \|\mZ_a^{(t)}(\mX_p) \mW_{U}^{(t)} - \mZ^{(t)}(\mX_p) \mW_{U}^{(t)}\|_F^2)\nonumber\\
    &\!\!\!\! \leq \!\!\!\!& \frac{729\|\mW_2^{(0)}\|^2\|\mW_U^{(0)}\|^2}{16}(\|\mZ^{(0)}(\mX_p)\|^2\|\mW_{1,a}^{(t)} - \mW_{1}^{(t)} \|_F^2 + \|\mW_1^{(0)}\|^2 \|\mZ_a^{(t)}(\mX_p) - \mZ^{(t)}(\mX_p) \|_F^2)\nonumber\\
    &\!\!\!\!\!\!\!\!& + 9\|\mZ^{(0)}(\mX_p)\|^2\|\mW_{U,a}^{(t)} - \mW_{U}^{(t)}\|_F^2 + 9\|\mW_U^{(0)}\|^2 \|\mZ_a^{(t)}(\mX_p) - \mZ^{(t)}(\mX_p) \|_F^2\nonumber\\
    &\!\!\!\! \leq \!\!\!\!& \frac{729\|\mW_2^{(0)}\|^2\|\mW_U^{(0)}\|^2\|\mZ^{(0)}(\mX_p)\|^2}{16}\|\mW_{1,a}^{(t)} - \mW_{1}^{(t)} \|_F^2 + 9\|\mZ^{(0)}(\mX_p)\|^2\|\mW_{U,a}^{(t)} - \mW_{U}^{(t)}\|_F^2\nonumber\\
    &\!\!\!\!\!\!\!\!& + \bigg(\frac{729\|\mW_1^{(0)}\|^2\|\mW_2^{(0)}\|^2\|\mW_U^{(0)}\|^2}{16} + 9\|\mZ^{(0)}(\mX_p)\|^2 \bigg)\nonumber\\
    &\!\!\!\!\!\!\!\!&\cdot\bigg(\frac{243M\sqrt{P}\|\mX_p\|^3(\max_{i,p}\|\mX_p(i,:)\|_2\|\mX_p\|^2)\|\mW_V^{(0)}\|^2\|\mW_Q^{(0)}\|^2}{4d_{QK}}\|\mW_{K,a}^{(t)} - \mW_K^{(t)} \|_F^2\nonumber\\
    &\!\!\!\!\!\!\!\!&+ \frac{243M\sqrt{P}\|\mX_p\|^3(\max_{i,p}\|\mX_p(i,:)\|_2\|\mX_p\|^2)\|\mW_V^{(0)}\|^2\|\mW_K^{(0)}\|^2}{4d_{QK}}\|\mW_{Q,a}^{(t)} - \mW_Q^{(t)} \|_F^2\nonumber\\
    &\!\!\!\!\!\!\!\!& +3 M\sqrt{P}(\max_{p}\|\mX_p\|)^2\|\mW_{V,a}^{(t)} - \mW_V^{(t)} \|_F^2\bigg)\nonumber\\
    &\!\!\!\! \leq \!\!\!\!& C_F \|\vtheta^{(t)}_a  - \vtheta^{(t)}\|_2^2,
\end{eqnarray}
where we define $    C_F = \max \big\{
\frac{729\|\mW_2^{(0)}\|^2\|\mW_U^{(0)}\|^2\max_{p}\|\mZ^{(0)}(\mX_p)\|^2}{16},  9\max_{p}\|\mZ^{(0)}(\mX_p)\|^2, \\   \big(\frac{729\|\mW_1^{(0)}\|^2\|\mW_2^{(0)}\|^2\|\mW_U^{(0)}\|^2}{16} + 9\max_{p}\|\mZ^{(0)}(\mX_p)\|^2 \big)\frac{243M\sqrt{P}\max_{p}\|\mX_p\|^3(\max_{i,p}\|\mX_p(i,:)\|_2\|\mX_p\|^2)\|\mW_V^{(0)}\|^2\|\mW_Q^{(0)}\|^2}{4d_{QK}}, \\   \big(\frac{729\|\mW_1^{(0)}\|^2\|\mW_2^{(0)}\|^2\|\mW_U^{(0)}\|^2}{16} + 9\max_{p}\|\mZ^{(0)}(\mX_p)\|^2 \big) \frac{243M\sqrt{P}\max_{p}\|\mX_p\|^3(\max_{i,p}\|\mX_p(i,:)\|_2\|\mX_p\|^2)\|\mW_V^{(0)}\|^2\|\mW_K^{(0)}\|^2}{4d_{QK}}, \\
27 M\sqrt{P}(\max_{p}\|\mX_p\|)^2\big( \max_{p}\|\mZ^{(0)}(\mX_p)\|^2 + \frac{81\|\mW_1^{(0)}\|^2\|\mW_2^{(0)}\|^2\|\mW_U^{(0)}\|^2}{16}  \big) \big\}$.

In addition, since it has been previously established that $\|\mW_{b,a}^{(t)} - \mW_b^{(t)} \|_F^2\leq \frac{9\sigma_{\min}^2(\mW_b^{(t)})}{4}, b=1,2,U,V,Q,K$, we can further obtain
\begin{eqnarray}
    \label{tigher bound of F difference}
    &\!\!\!\!\!\!\!\!&\|F_{\mTheta, a}^{(t)}(\mX_p) - F_{\mTheta}^{(t)}(\mX_p) \|_F^2\nonumber\\
    &\!\!\!\! \leq \!\!\!\!& \frac{6561\|\mW_2^{(0)}\|^2\|\mW_U^{(0)}\|^2\max_{p}\|\mZ^{(0)}(\mX_p)\|^2\sigma_{\min}^2(\mW_1^{(0)})}{64} + \frac{81\max_{p}\|\mZ^{(0)}(\mX_p)\|^2\sigma_{\min}^2(\mW_U^{(0)})}{4} \nonumber\\
    &\!\!\!\!\!\!\!\!& + \bigg(9\max_{p}\|\mZ^{(0)}(\mX_p)\|^2 + \frac{729\|\mW_1^{(0)}\|^2\|\mW_2^{(0)}\|^2\|\mW_U^{(0)}\|^2}{16}  \bigg)\bigg(\frac{27 M\sqrt{P}(\max_{p}\|\mX_p\|)^2\sigma_{\min}^2(\mW_V^{(0)})}{4}\nonumber\\
    &\!\!\!\!\!\!\!\!& +\frac{2187M\sqrt{P}\max_{p}\|\mX_p\|^3(\max_{i,p}\|\mX_p(i,:)\|_2\|\mX_p\|^2)\|\mW_V^{(0)}\|^2\|\mW_Q^{(0)}\|^2\sigma_{\min}^2(\mW_K^{(0)})}{16d_{QK}}\nonumber\\
    &\!\!\!\!\!\!\!\!& + \frac{2187M\sqrt{P}\max_{p}\|\mX_p\|^3(\max_{i,p}\|\mX_p(i,:)\|_2\|\mX_p\|^2)\|\mW_V^{(0)}\|^2\|\mW_K^{(0)}\|^2\sigma_{\min}^2(\mW_Q^{(0)})}{16d_{QK}} \bigg)\nonumber\\
    &\!\!\!\! := \!\!\!\!& C_2.
\end{eqnarray}

Considering that
\begin{eqnarray}
    \label{expansion of theta transformer}
    \|\nabla_{\vtheta}\Phi(\vtheta^{(t)}_a ) - \nabla_{\vtheta}\Phi(\vtheta^{(t)}) \|_2^2  &\!\!\!\!=\!\!\!\!& \hspace{-0.5cm} \sum_{b = 1,2,U,V,Q,K} \hspace{-0.5cm}\|\nabla_{\mW_b} L(\mW_{1,a}^{(t)}, \mW_{2,a}^{(t)},\mW_{Q,a}^{(t)}, \mW_{K,a}^{(t)}, \mW_{V,a}^{(t)}, \mW_{U,a}^{(t)})\nonumber\\
     &\!\!\!\!\!\!\!\!&- \nabla_{\mW_b} L(\mW_1^{(t)}, \mW_2^{(t)},\mW_Q^{(t)}, \mW_K^{(t)}, \mW_V^{(t)}, \mW_U^{(t)})\|_F^2,
\end{eqnarray}
we proceed to analyze each $\mW_b$ term individually for $b = 1,2,U,V,Q,K$. To simplify notation, we define $\nabla_{\mW_b} L(\mW_{1,a}^{(t)}, \mW_{2,a}^{(t)},\\ \mW_{Q,a}^{(t)}, \mW_{K,a}^{(t)}, \mW_{V,a}^{(t)}, \mW_{U,a}^{(t)}) = \nabla_{\mW_b} L(\{\mW_{b,a}^{(t)}  \}_b )$ and $\nabla_{\mW_b} L(\mW_1^{(t)}, \mW_2^{(t)},\mW_Q^{(t)}, \mW_K^{(t)}, \mW_V^{(t)}, \mW_U^{(t)}) = \nabla_{\mW_b} L(\{\mW_{b}^{(t)}  \}_b )$ for $b = 1,2,U,V,Q,K$.

Now, we begin by expanding the expression corresponding to $\mW_1$.
\begin{eqnarray}
    \label{Lipschitz constant of W1 in the transofrmer}
    &\!\!\!\!\!\!\!\!&\|\nabla_{\mW_1} L(\{\mW_{b,a}^{(t)}  \}_b ) - \nabla_{\mW_1} L(\{\mW_{b}^{(t)}  \}_b )\|_F^2\nonumber\\
    &\!\!\!\!=\!\!\!\!& \|\sum_{p=1}^{P}{\mZ_a^{(t)}(\mX_p)}^\top ( \phi_r'(\mZ_a^{(t)}(\mX_p)\mW_{1,a}^{(t)}) \odot ((F_{\mTheta,a}^{(t)}(\mX_p) - \mY_p){\mW_{U,a}^{(t)}}^\top{\mW_{2,a}^{(t)}}^\top))\nonumber\\
    &\!\!\!\!\!\!\!\!& - \sum_{p=1}^{P}{\mZ^{(t)}(\mX_p)}^\top ( \phi_r'(\mZ^{(t)}(\mX_p)\mW_1^{(t)}) \odot ((F_{\mTheta}^{(t)}(\mX_p) - \mY_p){\mW_{U}^{(t)}}^\top{\mW_{2}^{(t)}}^\top))   \|_F^2\nonumber\\
    &\!\!\!\!\leq \!\!\!\!& 3 \|\sum_{p=1}^{P} {\mZ_a^{(t)}(\mX_p)}^\top ( \phi_r'(\mZ_a^{(t)}(\mX_p)\mW_{1,a}^{(t)}) \odot ((F_{\mTheta,a}^{(t)}(\mX_p) - F_{\mTheta}^{(t)}(\mX_p)){\mW_{U,a}^{(t)}}^\top{\mW_{2,a}^{(t)}}^\top))\|_F^2\nonumber\\
    &\!\!\!\!\!\!\!\!& + 3 \| \sum_{p=1}^{P} {\mZ_a^{(t)}(\mX_p)}^\top ( \phi_r'(\mZ_a^{(t)}(\mX_p)\mW_{1,a}^{(t)}) \odot ((F_{\mTheta}^{(t)}(\mX_p) - \mY_p){\mW_{U,a}^{(t)}}^\top{\mW_{2,a}^{(t)}}^\top))\|_F^2\nonumber\\
    &\!\!\!\!\!\!\!\!& + 3\| \sum_{p=1}^{P}{\mZ^{(t)}(\mX_p)}^\top ( \phi_r'(\mZ^{(t)}(\mX_p)\mW_1^{(t)}) \odot ((F_{\mTheta}^{(t)}(\mX_p) - \mY_p){\mW_{U}^{(t)}}^\top{\mW_{2}^{(t)}}^\top))   \|_F^2\nonumber\\
    &\!\!\!\!\leq \!\!\!\!&  \frac{2187P }{64} \max_{p}\|\mZ^{(0)}(\mX_p)\|^2\|\mW_U^{(0)} \|^2 \|\mW_2^{(0)} \|^2 \sum_{p=1}^{P}\|F_{\mTheta, a}^{(t)}(\mX_p) - F_{\mTheta}^{(t)}(\mX_p)\|_F^2\nonumber\\
     &\!\!\!\!\!\!\!\!&+ \frac{2187P}{16} \max_{p}\|\mZ^{(0)}(\mX_p)\|^2\|\mW_U^{(0)} \|^2 \|\mW_2^{(0)} \|^2 \Phi(\vtheta^{(0)})\nonumber\\
    &\!\!\!\!\leq \!\!\!\!&  \frac{2187P C_F}{32} \max_{p}\|\mZ^{(0)}(\mX_p)\|^2\|\mW_U^{(0)} \|^2 \|\mW_2^{(0)} \|^2 \|\vtheta^{(t)}_a  - \vtheta^{(t)}\|_2^2,
\end{eqnarray}
where the last line follows \eqref{Lipschitz constant of F and Fa in the transofrmer} and $\Phi(\vtheta^{(0)})\leq 4\sum_{p=1}^{P}\|F_{\mTheta, a}^{(t)}(\mX_p) - F_{\mTheta}^{(t)}(\mX_p)\|_F^2 \leq 4PC_2$.

Similarly, for $\mW_2$, we have
\begin{eqnarray}
    \label{Lipschitz constant of W2 in the transofrmer}
    &\!\!\!\!\!\!\!\!&\|\nabla_{\mW_2} L(\{\mW_{b,a}^{(t)}  \}_b ) - \nabla_{\mW_2} L(\{\mW_{b}^{(t)}  \}_b )\|_F^2\nonumber\\
    &\!\!\!\!=\!\!\!\!& \|\sum_{p=1}^{P}(\phi_r(\mZ_a^{(t)}(\mX_p)\mW_{1,a}^{(t)}))^\top(F_{\mTheta,a}^{(t)}(\mX_p) - \mY_p){\mW_{U,a}^{(t)}}^\top\nonumber\\
    &\!\!\!\!\!\!\!\!&- \sum_{p=1}^{P}(\phi_r(\mZ^{(t)}(\mX_p)\mW_{1}^{(t)}))^\top(F_{\mTheta}^{(t)}(\mX_p) - \mY_p){\mW_{U}^{(t)}}^\top  \|_F^2\nonumber\\
    &\!\!\!\!\leq \!\!\!\!& 3\|\sum_{p=1}^{P}(\phi_r(\mZ_a^{(t)}(\mX_p)\mW_{1,a}^{(t)}))^\top(F_{\mTheta,a}^{(t)}(\mX_p) - F_{\mTheta}^{(t)}(\mX_p)){\mW_{U,a}^{(t)}}^\top \|_F^2\nonumber\\
    &\!\!\!\!\!\!\!\!&+ 3\|\sum_{p=1}^{P}(\phi_r(\mZ_a^{(t)}(\mX_p)\mW_{1,a}^{(t)}))^\top( F_{\mTheta}^{(t)}(\mX_p) - \mY_p){\mW_{U,a}^{(t)}}^\top \|_F^2\nonumber\\
    &\!\!\!\!\!\!\!\!& + 3\|\sum_{p=1}^{P}(\phi_r(\mZ^{(t)}(\mX_p)\mW_{1}^{(t)}))^\top(F_{\mTheta}^{(t)}(\mX_p) - \mY_p){\mW_{U}^{(t)}}^\top  \|_F^2\nonumber\\
    &\!\!\!\!\leq \!\!\!\!& 3P\sum_{p=1}^{P}\|(\mZ_a^{(t)}(\mX_p)\mW_{1,a}^{(t)})^\top(F_{\mTheta,a}^{(t)}(\mX_p) - F_{\mTheta}^{(t)}(\mX_p)){\mW_{U,a}^{(t)}}^\top \|_F^2\nonumber\\
     &\!\!\!\!\!\!\!\!&+ 3P\sum_{p=1}^{P}\|(\mZ_a^{(t)}(\mX_p)\mW_{1,a}^{(t)})^\top( F_{\mTheta}^{(t)}(\mX_p) - \mY_p){\mW_{U,a}^{(t)}}^\top \|_F^2\nonumber\\
    &\!\!\!\!\!\!\!\!& + 3P\sum_{p=1}^{P}\|(\mZ^{(t)}(\mX_p)\mW_{1}^{(t)})^\top(F_{\mTheta}^{(t)}(\mX_p) - \mY_p){\mW_{U}^{(t)}}^\top  \|_F^2\nonumber\\
    &\!\!\!\!\leq \!\!\!\!&  \frac{2187P C_F}{32} \max_{p}\|\mZ^{(0)}(\mX_p)\|^2\|\mW_U^{(0)} \|^2 \|\mW_1^{(0)} \|^2 \|\vtheta^{(t)}_a  - \vtheta^{(t)}\|_2^2,
\end{eqnarray}
where the second inequality uses {Assumption} \ref{assumption of activation function}, and the last line follows \eqref{Lipschitz constant of F and Fa in the transofrmer} and $\Phi(\vtheta^{(0)})\leq 4\sum_{p=1}^{P}\|F_{\mTheta, a}^{(t)}(\mX_p) - F_{\mTheta}^{(t)}(\mX_p)\|_F^2 \leq 4PC_2$.

In addition, for $\mW_U$, we have
\begin{eqnarray}
    &\!\!\!\!\!\!\!\!&\|\nabla_{\mW_U} L(\{\mW_{b,a}^{(t)}  \}_b ) - \nabla_{\mW_U} L(\{\mW_{b}^{(t)}  \}_b )\|_F^2\nonumber\\
    &\!\!\!\!=\!\!\!\!& \|\sum_{p=1}^{P}(\phi_r(\mZ_a^{(t)}(\mX_p)\mW_{1,a}^{(t)})\mW_{2,a}^{(t)} +\mZ_a^{(t)}(\mX_p))^\top(F_{\mTheta,a}^{(t)}(\mX_p) - \mY_p)\nonumber\\
    &\!\!\!\!\!\!\!\!&- \sum_{p=1}^{P}(\phi_r(\mZ^{(t)}(\mX_p)\mW_{1}^{(t)})\mW_{2}^{(t)} +\mZ^{(t)}(\mX_p))^\top(F_{\mTheta}^{(t)}(\mX_p) - \mY_p)   \|_F^2\nonumber\\
    &\!\!\!\!\leq \!\!\!\!& 2\|\sum_{p=1}^{P}(\phi_r(\mZ_a^{(t)}(\mX_p)\mW_{1,a}^{(t)})\mW_{2,a}^{(t)} )^\top(F_{\mTheta,a}^{(t)}(\mX_p) - \mY_p)\nonumber\\
    &\!\!\!\!\!\!\!\!&- \sum_{p=1}^{P}(\phi_r(\mZ^{(t)}(\mX_p)\mW_{1}^{(t)})\mW_{2}^{(t)} )^\top(F_{\mTheta}^{(t)}(\mX_p) - \mY_p) \|_F^2\nonumber\\
    &\!\!\!\!\!\!\!\!& + 2\|\sum_{p=1}^{P}{\mZ_a^{(t)}(\mX_p)}^\top(F_{\mTheta,a}^{(t)}(\mX_p) - \mY_p) - \sum_{p=1}^{P}{\mZ^{(t)}(\mX_p)}^\top(F_{\mTheta}^{(t)}(\mX_p) - \mY_p)  \|_F^2\nonumber
    \end{eqnarray}
    \begin{eqnarray}
    \label{Lipschitz constant of WU in the transofrmer}
    &\!\!\!\!\leq \!\!\!\!& 6 \|\sum_{p=1}^{P}(\phi_r(\mZ_a^{(t)}(\mX_p)\mW_{1,a}^{(t)})\mW_{2,a}^{(t)} )^\top(F_{\mTheta,a}^{(t)}(\mX_p) - F_{\mTheta}^{(t)}(\mX_p))\|_F^2\nonumber\\
    &\!\!\!\!\!\!\!\!&+ 6 \|\sum_{p=1}^{P}(\phi_r(\mZ_a^{(t)}(\mX_p)\mW_{1,a}^{(t)})\mW_{2,a}^{(t)} )^\top(F_{\mTheta}^{(t)}(\mX_p) - \mY_p)\|_F^2\nonumber\\
    &\!\!\!\!\!\!\!\!& + 6 \|\sum_{p=1}^{P}(\phi_r(\mZ^{(t)}(\mX_p)\mW_{1}^{(t)})\mW_{2}^{(t)} )^\top(F_{\mTheta}^{(t)}(\mX_p) - \mY_p) \|_F^2 + 6\|\sum_{p=1}^{P}{\mZ_a^{(t)}(\mX_p)}^\top(F_{\mTheta,a}^{(t)}(\mX_p) - F_{\mTheta}^{(t)}(\mX_p)) \|_F^2\nonumber\\
    &\!\!\!\!\!\!\!\!& + 6\|\sum_{p=1}^{P}{\mZ_a^{(t)}(\mX_p)}^\top(F_{\mTheta}^{(t)}(\mX_p) - \mY_p) \|_F^2 + 6\|\sum_{p=1}^{P}{\mZ^{(t)}(\mX_p)}^\top(F_{\mTheta}^{(t)}(\mX_p) - \mY_p) \|_F^2\nonumber\\
    &\!\!\!\!\leq \!\!\!\!& 6P \sum_{p=1}^{P}\|(\mZ_a^{(t)}(\mX_p)\mW_{1,a}^{(t)}\mW_{2,a}^{(t)} )^\top(F_{\mTheta,a}^{(t)}(\mX_p) - F_{\mTheta}^{(t)}(\mX_p))\|_F^2\nonumber\\
    &\!\!\!\!\!\!\!\!& + 6P \sum_{p=1}^{P} \|(\mZ_a^{(t)}(\mX_p)\mW_{1,a}^{(t)}\mW_{2,a}^{(t)} )^\top(F_{\mTheta}^{(t)}(\mX_p) - \mY_p)\|_F^2\nonumber\\
    &\!\!\!\!\!\!\!\!& + 6P \sum_{p=1}^{P} \|(\mZ^{(t)}(\mX_p)\mW_{1}^{(t)}\mW_{2}^{(t)} )^\top(F_{\mTheta}^{(t)}(\mX_p) - \mY_p) \|_F^2 + 6P \sum_{p=1}^{P}\|{\mZ_a^{(t)}(\mX_p)}^\top(F_{\mTheta,a}^{(t)}(\mX_p) - F_{\mTheta}^{(t)}(\mX_p)) \|_F^2\nonumber\\
    &\!\!\!\!\!\!\!\!& + 6P \sum_{p=1}^{P}\|{\mZ_a^{(t)}(\mX_p)}^\top(F_{\mTheta}^{(t)}(\mX_p) - \mY_p) \|_F^2 + 6P \sum_{p=1}^{P}\|{\mZ^{(t)}(\mX_p)}^\top(F_{\mTheta}^{(t)}(\mX_p) - \mY_p) \|_F^2 \nonumber\\
    &\!\!\!\!\leq \!\!\!\!&\frac{2187}{32}P \sum_{p=1}^{P}\|\mW_1^{(0)}\|^2\|\mW_2^{(0)}\|^2\|\mZ^{(0)}(\mX_p)\|^2 \|F_{\mTheta,a}^{(t)}(\mX_p) - F_{\mTheta}^{(t)}(\mX_p)\|_F^2\nonumber\\
    &\!\!\!\!\!\!\!\!&+ \frac{2187}{16}P \sum_{p=1}^{P}\|\mW_1^{(0)}\|^2\|\mW_2^{(0)}\|^2\|\mZ^{(0)}(\mX_p)\|^2 \|F_{\mTheta}^{(t)}(\mX_p) - \mY_p\|_F^2\nonumber\\
    &\!\!\!\!\!\!\!\!& +\frac{27}{2}P \sum_{p=1}^{P}\|\mZ^{(0)}(\mX_p)\|^2 \|F_{\mTheta,a}^{(t)}(\mX_p) - F_{\mTheta}^{(t)}(\mX_p)\|_F^2 + 27P \sum_{p=1}^{P}\|\mZ^{(0)}(\mX_p)\|^2\|F_{\mTheta}^{(t)}(\mX_p) - \mY_p\|_F^2\nonumber\\
    &\!\!\!\!\leq \!\!\!\!&  \frac{2187PC_F}{16}\|\mW_1^{(0)}\|^2\|\mW_2^{(0)}\|^2\max_{p}\|\mZ^{(0)}(\mX_p)\|^2\|\vtheta^{(t)}_a  - \vtheta^{(t)}\|_2^2 + 27PC_F\max_{p}\|\mZ^{(0)}(\mX_p)\|^2\|\vtheta^{(t)}_a  - \vtheta^{(t)}\|_2^2,\nonumber\\
\end{eqnarray}
where the last line follows \eqref{Lipschitz constant of F and Fa in the transofrmer} and  $\Phi(\vtheta^{(0)})\leq 4\sum_{p=1}^{P}\|F_{\mTheta, a}^{(t)}(\mX_p) - F_{\mTheta}^{(t)}(\mX_p)\|_F^2 \leq 4PC_2$.

Similarly, for $\mW_V$, $\mW_Q$, and $\mW_K$, by invoking \eqref{Lipschitz constant of F and Fa in the transofrmer} and  $\Phi(\vtheta^{(0)})\leq 4\sum_{p=1}^{P}\|F_{\mTheta, a}^{(t)}(\mX_p) - F_{\mTheta}^{(t)}(\mX_p)\|_F^2 \leq 4PC_2$, we have
\begin{eqnarray}
    \label{Lipschitz constant of WV in the transofrmer}
    &\!\!\!\!\!\!\!\!&\|\nabla_{\mW_V} L(\{\mW_{b,a}^{(t)}  \}_b ) - \nabla_{\mW_V} L(\{\mW_{b}^{(t)}  \}_b )\|_F^2\nonumber\\
    &\!\!\!\!\leq\!\!\!\!& \frac{177147PC_F}{256d_{QK}} (\max_p\|\mX_p\|)^6\|\mW_1^{(0)}\|^2\|\mW_2^{(0)}\|^2\|\mW_U^{(0)}\|^2\|\mW_Q^{(0)}\|^2\|\mW_K^{(0)}\|^2\|\vtheta^{(t)}_a  - \vtheta^{(t)}\|_2^2\nonumber\\
    &\!\!\!\!\!\!\!\!& + \frac{2187PC_F}{16d_{QK}}(\max_p\|\mX_p\|)^6 \|\mW_U^{(0)}\|^2\|\mW_Q^{(0)}\|^2\|\mW_K^{(0)}\|^2\|\vtheta^{(t)}_a  - \vtheta^{(t)}\|_2^2,\\
    \label{Lipschitz constant of WQ in the transofrmer}
    &\!\!\!\!\!\!\!\!&\|\nabla_{\mW_Q} L(\{\mW_{b,a}^{(t)}  \}_b ) - \nabla_{\mW_Q} L(\{\mW_{b}^{(t)}  \}_b )\|_F^2\nonumber\\
    &\!\!\!\!=\!\!\!\!& \frac{2187PC_F M}{4}\max_{i,p}\|\mX_p(i,:)\|_2^2(\max_p\|\mX_p\|)^4 \|\mW_U^{(0)}\|^2\|\mW_V^{(0)}\|^2\|\mW_K^{(0)}\|^2\|\vtheta^{(t)}_a  - \vtheta^{(t)}\|_2^2\nonumber\\
    &\!\!\!\!\!\!\!\!& + \frac{177147PC_F M}{64}\max_{i,p}\|\mX_p(i,:)\|_2^2(\max_p\|\mX_p\|)^4 \|\mW_1^{(0)}\|^2\|\mW_2^{(0)}\|^2\|\mW_U^{(0)}\|^2\nonumber\\
    &\!\!\!\!\!\!\!\!&\cdot\|\mW_V^{(0)}\|^2\|\mW_K^{(0)}\|^2\|\vtheta^{(t)}_a  - \vtheta^{(t)}\|_2^2, 
    \end{eqnarray}
    \begin{eqnarray}
    \label{Lipschitz constant of WK in the transofrmer}
    &\!\!\!\!\!\!\!\!&\|\nabla_{\mW_K} L(\{\mW_{b,a}^{(t)}  \}_b ) - \nabla_{\mW_K} L(\{\mW_{b}^{(t)}  \}_b )\|_F^2\nonumber\\
    &\!\!\!\!=\!\!\!\!& \frac{2187PC_F M}{4}\max_{i,p}\|\mX_p(i,:)\|_2^2(\max_p\|\mX_p\|)^4 \|\mW_U^{(0)}\|^2\|\mW_V^{(0)}\|^2\|\mW_Q^{(0)}\|^2\|\vtheta^{(t)}_a  - \vtheta^{(t)}\|_2^2\nonumber\\
    &\!\!\!\!\!\!\!\!& + \frac{177147PC_F M}{64}\max_{i,p}\|\mX_p(i,:)\|_2^2(\max_p\|\mX_p\|)^4 \|\mW_1^{(0)}\|^2\|\mW_2^{(0)}\|^2\|\mW_U^{(0)}\|^2\nonumber\\
    &\!\!\!\!\!\!\!\!&\cdot\|\mW_V^{(0)}\|^2\|\mW_Q^{(0)}\|^2\|\vtheta^{(t)}_a  - \vtheta^{(t)}\|_2^2.
\end{eqnarray}

Combing \eqref{Lipschitz constant of W1 in the transofrmer}-\eqref{Lipschitz constant of WK in the transofrmer} with \eqref{expansion of theta transformer}, we have $\|\Phi(\vtheta^{(t)}_a ) - \Phi(\vtheta^{(t)}) \|_2\leq C \|\vtheta^{(t)}_a  - \vtheta^{(t)}\|_2$, where $    C^2  =  \frac{2187 PC_F}{32} \max_{p}\|\mZ^{(0)}(\mX_p)\|^2  \cdot\|\mW_U^{(0)} \|^2 (\|\mW_1^{(0)} \|^2 + \|\mW_2^{(0)} \|^2) + PC_F \max_{p}\|\mZ^{(0)}(\mX_p)\|^2 \big(\frac{2187}{16}\|\mW_1^{(0)}\|^2\|\mW_2^{(0)}\|^2 + 27 \big)  + \frac{PC_F}{d_{QK}} (\max_{p}\|\mX_p\|)^6 \|\mW_U^{(0)}\|^2\\ \cdot  \|\mW_Q^{(0)}\|^2  \|\mW_K^{(0)}\|^2\big(\frac{2187}{16} + \frac{177147}{256} \|\mW_1^{(0)}\|^2\|\mW_2^{(0)}\|^2\big)  + \frac{2187PC_F M}{4}\max_{i,p}\|\mX_p(i,:)\|_2^2(\max_p\|\mX_p\|)^4 \|\mW_U^{(0)}\|^2\|\mW_V^{(0)}\|^2\\ \cdot(\|\mW_K^{(0)}\|^2   + \|\mW_Q^{(0)}\|^2 )  + \frac{177147PC_F M}{64}\max_{i,p}\|\mX_p(i,:)\|_2^2(\max_p\|\mX_p\|)^4 \|\mW_1^{(0)}\|^2   \|\mW_2^{(0)}\|^2 \|\mW_U^{(0)}\|^2\|\mW_V^{(0)}\|^2 (\|\mW_K^{(0)}\|^2 + \|\mW_Q^{(0)}\|^2 )$.

\paragraph{Step III: Proof of last inequality in \eqref{Summary of Proof requirements transformer}. } Using \Cref{convex property of loss} with the Lipschitz continuity of the gradient, we can analyze the convergence property of $\vtheta^{(t+1)} = \vtheta^{(t)} - \mu \nabla_{\vtheta}\Phi(\vtheta^{(t)})$ as follows:
\begin{eqnarray}
    \label{relationship between different time in the transformer}
    \Phi(\vtheta^{(t+1)}) &\!\!\!\!\leq\!\!\!\!&  \Phi(\vtheta^{(t)}) + \< \nabla_{\vtheta}\Phi(\vtheta^{(t)}), \vtheta^{(t+1)} - \vtheta^{(t)} \> +\frac{C}{2}\|\vtheta^{(t+1)} - \vtheta^{(t)} \|_2^2\nonumber\\
    &\!\!\!\!=\!\!\!\!& \Phi(\vtheta^{(t)}) - \mu\|\nabla_{\vtheta}\Phi(\vtheta^{(t)})\|_2^2 +\frac{C\mu^2}{2}\|\nabla_{\vtheta}\Phi(\vtheta^{(t)})\|_2^2\nonumber\\
    &\!\!\!\!\leq\!\!\!\!&  \Phi(\vtheta^{(t)})  - \frac{\mu}{2}\|\nabla_{\vtheta}\Phi(\vtheta^{(t)})\|_2^2\nonumber\\
    &\!\!\!\!\leq\!\!\!\!& \Phi(\vtheta^{(t)})  - \frac{\mu}{2}\|\nabla_{\mW_2} L(\mW_1^{(t)}, \mW_2^{(t)},\mW_Q^{(t)}, \mW_K^{(t)}, \mW_V^{(t)}, \mW_U^{(t)})\|_F^2\nonumber\\
    &\!\!\!\!\leq\!\!\!\!& \Phi(\vtheta^{(t)}) - \mu \frac{\sigma_{\min}^2(\mW_U^{(0)})\min_{p}\sigma_{\min}^2(\phi_r(\mZ^{(0)}(\mX_p) \mW_1^{(0)}))}{16}\Phi(\vtheta^{(t)})\nonumber\\
    &\!\!\!\!=\!\!\!\!& (1 - \mu \alpha)\Phi(\vtheta^{(t)}),
\end{eqnarray}
where the second and  third inequalities respectively follows $\mu \leq \frac{1}{C}$ and $\|\nabla_{\vtheta}\Phi(\vtheta)\|_2^2 = \|\nabla_{\mW_1} L\|_F^2 + \|\nabla_{\mW_2} L\|_F^2 + \|\nabla_{\mW_U} L\|_F^2 + \|\nabla_{\mW_V} L\|_F^2 + \|\nabla_{\mW_Q} L\|_F^2 + \|\nabla_{\mW_K} L\|_F^2$. The fourth inequality uses $\|\nabla_{\mW_2} L(\mW_1^{(t)}, \mW_2^{(t)},\mW_Q^{(t)}, \mW_K^{(t)},  \mW_V^{(t)}, \mW_U^{(t)}) \|_F = \|\sum_{p=1}^P(\phi_r(\mZ^{(t)}(\mX_p)\mW_1^{(t)}))^\top(F_{\mTheta}^{(t)}(\mX_p) - \mY_p){\mW_U^{(t)}}^\top\|_F = \|\vphi_P^{(t)} (\ol F_{\mTheta}^{(t)}(\mX) - \ol \mY){\mW_U^{(t)}}^\top\|_F\geq \frac{1}{4}\sigma_{\min}^2(\mW_U^{(0)})\sigma_{\min}^2(\vphi_P^{(0)})\|\ol F_{\mTheta}^{(t)}(\mX)  - \ol \mY\|_F$, where  $\vphi_P^{(t)} = [(\phi_r(\mZ^{(t)}(\mX_1)\mW_1^{(t)}))^\top  \cdots   (\phi_r(\mZ^{(t)}(\mX_P)\mW_1^{(t)}))^\top   ]$. Here we define $\alpha = \frac{\sigma_{\min}^2(\mW_U^{(0)})\sigma_{\min}^2(\vphi_P^{(0)})}{16}$.

\end{proof}

\section{Proof of \Cref{Theorem of convergence analysis of transformer model simplified}}
\label{proof of convergence rate in the tranformer model simplified}

\begin{proof}
To simplify \Cref{requirement of the initialization transformer theorem  detailed} in {Appendix}~\ref{detailed version of convergence rate in the tranformer model simplified}, we define $\ol \lambda = \max_{b=1,2,U,V,Q,K} \|\mW_b^{(0)}\|$ and $\underline{\lambda} = \min_{b=1,2,U,V,Q,K}\sigma_{\min}(\mW_b^{(0)})$. Considering that when $P$ is sufficiently large, $C_2$ in \eqref{requirement of the initialization transformer theorem detailed} becomes significantly larger than the other terms, we can neglect this term. Consequently, the initialization requirement and the constant $C$ can be respectively simplified as
\begin{eqnarray}
    \label{simplified initialization requirement appendix}
    &\!\!\!\!\!\!\!\!&\Phi^{\frac{1}{2}}(\vtheta^{(0)}) \lesssim \min\bigg\{ \frac{\underline{\lambda}^2 \min_{p}\sigma_{\min}(\phi_r(\mZ^{(0)}(\mX_p) \mW_1^{(0)}))}{\max\{1, \max_{p}\|\mX_p\| \}\cdot \max\{1,\ol\lambda^2  \}\cdot\max_{p}\|\mZ^{(0)}(\mX_p)\|}, \frac{\underline{\lambda}^2 \min_{p}\sigma_{\min}(\phi_r(\mZ^{(0)}(\mX_p) \mW_1^{(0)}))}{\max_{p}\|\mX_p\|^2\max_{i.p}\|\mX_p(i,:) \|_2\max\{ \ol \lambda^3, \ol \lambda^5 \} }, \nonumber\\
     &\!\!\!\!\!\!\!\!& \frac{\underline{\lambda} \min_{p}\sigma_{\min}(\phi_r(\mZ^{(0)}(\mX_p) \mW_1^{(0)}))\cdot\min\big\{\min_{p}\sigma_{\min}(\mZ^{(0)}(\mX_p)),  \min_{p}\sigma_{\min}^2(\phi_r(\mZ^{(0)}(\mX_p) \mW_1^{(0)})), \sigma_{\min}(\vphi_P^{(0)})\big\}}{\max \big\{\max_{p}\|\mX_p\|^5\max_{i,p}\|\mX_p(i,:) \|_2\cdot \max\{\ol \lambda^5,\ol \lambda^7 \},  \max_{p}\|\mX_p\|^2\ol \lambda\max\{1,\ol \lambda^2 \}, \max_{p}\|\mZ^{(0)}(\mX_p)\|^2\ol \lambda^2  \big\} } \bigg\}\nonumber\\
     &\!\!\!\!=\!\!\!\!&\ol C \frac{\underline{\lambda}^2 (\min_{p}\sigma_{\min}(\phi_r(\mZ^{(0)}(\mX_p) \mW_1^{(0)})))^2}{\max\{1,\ol\lambda^7 \}\cdot \max_{p}\|\mX_p\|^5\max_{i,p}\|\mX_p(i,:) \|_2 \max_{p}\|\mZ^{(0)}(\mX_p)\|^2}\nonumber
\end{eqnarray}
and
\begin{eqnarray}
    \label{constant C simplify appendix}
    C &\!\!\!\!=\!\!\!\!& \wt C \cdot \max\{\ol\lambda^2\max_{p}\|\mZ^{(0)}(\mX_p)\|, \max_{p}\|\mZ^{(0)}(\mX_p)\|, \max_{p}\|\mX_p\|(\ol\lambda^3 + \max_{p}\|\mZ^{(0)}(\mX_p)\|),  \nonumber\\
    &\!\!\!\!\!\!\!\!&(\ol\lambda^3 + \max_{p}\|\mZ^{(0)}(\mX_p)\|)\max_{p}\|\mX_p\|^3\ol \lambda^2  \}(\max_{p}\|\mZ^{(0)}(\mX_p)\|(1 + \ol \lambda^2) \nonumber\\
    &\!\!\!\!\!\!\!\!&+  \max_{p}\|\mX_p\|^3 (\ol \lambda^3 + \ol \lambda^5) + \max_{p}\|\mX_p\|^2\max_{i,p}\|\mX_p(i,:) \|_2(\ol \lambda^3 + \ol \lambda^5)   ),
\end{eqnarray}
where $\ol C$ and $\wt C$ are positive constants.

This completes the proof.
\end{proof}

\section{Proof of \Cref{Theorem of convergence analysis for theta of transformer model}}
\label{appendix:convergence to the global minimum}

In this section, we will prove $\{\vtheta^{(t)} \}_{t=1}^{\infty}$ is a Cauchy sequence. Let us fix any $\epsilon > 0$. We need to show that there exists $z > 0$ such that for every $i, j \geq z$, $\|\vtheta^{(j)} - \vtheta^{(i)}\|_2 < \epsilon$. Without loss of generality, we assume that $i < j$. Then, we have
\begin{eqnarray}
    \label{bound of theta difference}
    &\!\!\!\!\!\!\!\!&\|\vtheta^{(j)} - \vtheta^{(i)}\|_2\nonumber\\
    &\!\!\!\!=\!\!\!\!& \sqrt{\sum_{b=1,2,U,V,Q,K}\|\mW_b^{(j)} - \mW_b^{(i)} \|_F^2 }\nonumber\\
    &\!\!\!\!\leq \!\!\!\!& \sum_{b=1,2,U,V,Q,K}\|\mW_b^{(j)} - \mW_b^{(i)} \|_F\nonumber\\
    &\!\!\!\!\leq \!\!\!\!& \sum_{b=1,2,U,V,Q,K}\sum_{s=i}^{j-1}\|\mW_b^{(s+1)} - \mW_b^{(s)} \|_F\nonumber\\
    &\!\!\!\!\leq \!\!\!\!&\sum_{b=1,2,U,V,Q,K}\sum_{s=i}^{j-1} \mu\|\nabla_{\mW_b} L(\mW_{1}^{(s)}, \mW_{2}^{(s)},\mW_{Q}^{(s)}, \mW_{K}^{(s)}, \mW_{V}^{(s)}, \mW_{U}^{(s)}) \|_F\nonumber\\
    &\!\!\!\!\leq \!\!\!\!& (1 - \mu\alpha)^{\frac{i}{2}} \frac{1 - (1 - \mu \alpha)^{\frac{j-i}{2}}}{1 - (1 - \mu \alpha)^{\frac{1}{2}}}\mu C_W \Phi^{\frac{1}{2}}(\vtheta^{(0)})\nonumber\\
    &\!\!\!\!\leq \!\!\!\!& (1 - \mu\alpha)^{\frac{i}{2}}\frac{ C_W}{\alpha} \Phi^{\frac{1}{2}}(\vtheta^{(0)}),
\end{eqnarray}
where we define $C_W = \frac{27\sqrt{2P}}{8}\max_{p}\|\mZ^{(0)}(\mX_p)\| \|\mW_U^{(0)}\|(\|\mW_1^{(0)}\| + \|\mW_2^{(0)}\|) + \frac{27\sqrt{2P}}{8}(1 + \|\mW_1^{(0)}\|  \|\mW_2^{(0)}\| )(\max_{p}\|\mZ^{(0)}(\mX_p)\| + \sqrt{M}\max_{p}\|\mX_p\|\|\mW_U^{(0)}\| )  + \frac{243\sqrt{2MP}}{16}\max_{i,p}\|\mX(i,:)\|_2\|\mX_p\|^2(1 + \|\mW_1^{(0)}\|\|\mW_2^{(0)}\|) \|\mW_U^{(0)}\|  \|\mW_V^{(0)}\|(\|\mW_K^{(0)}\| + \|\mW_Q^{(0)}\|)$ and fourth inequality follows \eqref{upper bound W1 transformer}, \eqref{upper bound W2 transformer}-\eqref{upper bound WK transformer}. In addition, the last line uses $\mu\frac{1 - (1 - \mu \alpha)^{\frac{j-i}{2}}}{1 - (1 - \mu \alpha)^{\frac{1}{2}}} =\frac{(1 - (1 - \mu\alpha)) }{\alpha} \frac{1 - (1 - \mu \alpha)^{\frac{j-i}{2}}}{1 - (1 - \mu \alpha)^{\frac{1}{2}}}\leq \frac{1}{\alpha}$.

Note that $(1 - \mu\alpha)^{\frac{i}{2}} \leq (1 - \mu\alpha)^{\frac{z}{2}}$ and thus there exists a sufficiently large $z$ such that $\|\vtheta^{(j)} - \vtheta^{(i)}\|_2 < \epsilon$. This shows that  $\{\vtheta^{(t)} \}_{t=1}^{\infty}$ is a Cauchy sequence, and hence convergent to some $\vtheta^\star$. By continuity, $\Phi(\vtheta^\star) = \Phi(\lim_{t \to \infty}\vtheta^{(t)}) = \lim_{t \to \infty}\Phi(\vtheta^{(t)}) = 0$, and thus $\vtheta^\star$ is a global minimizer. The rate of
convergence is
\begin{eqnarray}
    \label{convergence rate of theta transformer}
    \|\vtheta^{(k)} - \vtheta^\star\|_2 = \lim_{t \to \infty}\|\vtheta^{(k)} - \vtheta^{(t)}\|_2\leq (1 - \mu\alpha)^{\frac{k}{2}}\frac{ C_W}{\alpha} \Phi^{\frac{1}{2}}(\vtheta^{(0)}).
\end{eqnarray}

\section{Technical tools used in the proofs}
\label{Technical tools used in proofs}

\begin{lemma}(\cite[Lemma 4.3]{nguyen2020global})
\label{convex property of loss}
Let $f : \R^n \to \R$ be a twice continuously differentiable function. Let $x, y \in \R^n$ be given, and assume that $\|\nabla f(z) - \nabla f(x)\|_2 \leq C \|z - x\|_2$ for every $z = x + t(y - x)$ with $t \in [0, 1]$. Then,
\begin{eqnarray}
    \label{property of loss function}
    f(y) \leq f(x) + \langle \nabla f(x), y - x \rangle + \frac{C}{2} \|x - y\|^2.
\end{eqnarray}
\end{lemma}

\begin{lemma}
\label{property of softmax function and its derivative}
For any matrix $\mA\in\R^{d_1\times d_2}$, we apply the row-wise softmax function $\phi_s(\cdot)$ to obtain $\phi_s(\mA)\in\R^{d_1\times d_2}$, where each row is given by
\begin{eqnarray}
    \label{each row in the softmax function}
    \phi_s(\mA(i,:)) = \begin{bmatrix} \frac{e^{\mA(i,1)}}{\sum_{j=1}^{d_2} e^{\mA(i,j)}} \ \cdots \ \frac{e^{\mA(i,d_2)}}{\sum_{j=1}^{d_2} e^{\mA(i,j)}} \end{bmatrix}.
\end{eqnarray}
First, we have
\begin{eqnarray}
    \label{property in the softmax function1}
    1\leq \|\phi_s(\mA)\|_F^2 \leq d_1
\end{eqnarray}
and
\begin{eqnarray}
    \label{property in the softmax function2}
    \|\phi_s(\mA_1) - \phi_s(\mA_2)\|_F^2 \leq 4d_1 \|\mA_1 - \mA_2\|_F^2,
\end{eqnarray}
where $\mA_1,\mA_2\in\R^{d_1\times d_2}$.

Second, the Jacobian $\phi_s'(\mA(i,:)), i=1,\dots, d_1$ of the softmax can be written as
\begin{eqnarray}
    \label{first derivative of softmax function}
    \phi_s'(\mA(i,:)) = \text{diag}(\phi_s(\mA(i,:))) - \phi_s^\top(\mA(i,:))\phi_s(\mA(i,:))\in\R^{d_2\times d_2},
\end{eqnarray}
this further implies $\|\phi_s'(\mA(i,:))\|_F\leq 2$.
\end{lemma}

\begin{proof}
According to \cite[Lemma 8]{wu2023convergence}, we have $\frac{1}{\sqrt{d_1}}\leq \|\phi_s(\mA(i,:))\|_2 \leq 1$ and further obtain
\begin{eqnarray}
    \label{property in the softmax function1 expansion}
    1\leq \|\phi_s(\mA)\|_F^2 \leq d_1.
\end{eqnarray}

In addition, we can derive
\begin{eqnarray}
    \label{property in the softmax function2 expansion}
    \|\phi_s(\mA_1) - \phi_s(\mA_2)\|_F^2 &\!\!\!\!\leq \!\!\!\!& \|\phi_s(\mA_1) - \phi_s(\mA_2)\|_1^2\nonumber\\
    &\!\!\!\!\leq\!\!\!\!& 4(\sum_{i=1}^{d_1}\|\mA_1(i,:) - \mA_2(i,:) \|_{\infty} )^2\nonumber\\
    &\!\!\!\!\leq\!\!\!\!& 2d_1 \sum_{i=1}^{d_1}\|\mA_1(i,:) - \mA_2(i,:) \|_{\infty}^2\nonumber\\
    &\!\!\!\!\leq\!\!\!\!& 4d_1\|\mA_1 - \mA_2\|_F^2,
\end{eqnarray}
where the second inequality follows \cite[Corollary A.7]{edelman2022inductive}.

In the end, the Jacobian of the softmax has been derived from \cite[Lemma 11]{wu2023convergence}.
\end{proof}

\begin{lemma}(\cite[Corollary 5.35]{vershynin2010introduction})
For any matrix $\mW \in \R^{d_1 \times d_2}$ where $d_1 > d_2$ and each element is sampled independently from $\calN(0,1)$, for every $\zeta \geq 0$, with probability at least $1 - 2\exp(-\zeta^2/2)$ one has:
\begin{eqnarray}
    \label{upper and lower bounds of Gaussian matrix}
    \sqrt{d_1} - \sqrt{d_2} - \zeta \leq \phi_{\text{min}}(\mW) \leq \|\mW\| \leq \sqrt{d_1} + \sqrt{d_2} + \zeta.
\end{eqnarray}
\end{lemma}

As a direct consequence, we have
\begin{lemma}
\label{spectral norm of Gaussian matrix inequality}
For any matrix $\mW \in \R^{d_1 \times d_2}$ where $d_1/4 > d_2$ and each element is sampled independently from $\calN(0,\gamma^2)$, with probability at least $1 - 2\exp(-d_1/8)$, one has:
\begin{eqnarray}
    \label{upper and lower bounds of Gaussian matrix1}
    \gamma(\frac{\sqrt{d_1}}{2} - \sqrt{d_2})  \leq \phi_{\text{min}}(\mW) \leq \|\mW\| \leq \gamma(\frac{3\sqrt{d_1}}{2} + \sqrt{d_2}).
\end{eqnarray}
\end{lemma}

\begin{theorem}
(Hoeffding's inequality)
\label{Hoeffding's inequality theorem}
Let $Z_1, \dots, Z_n$ be independent bounded random variables with $Z_i \in [a, b]$ for all $i$, where $-\infty < a \leq b < \infty$. Then
\begin{eqnarray}
    \label{Hoeffding inequality}
    \mathbb{P} \left( \bigg|\frac{1}{n} \sum_{i=1}^{n} (Z_i - \mathbb{E}[Z_i])\bigg| \geq t \right) \leq 2\exp \left( - \frac{2 n t^2}{(b - a)^2} \right)
\end{eqnarray}
for all $t \geq 0$.

\end{theorem}

\end{document}